\documentclass{article}

\usepackage{microtype}
\usepackage{graphicx}
\usepackage{subcaption}
\usepackage{booktabs} 

\usepackage{hyperref}

\usepackage[preprint]{icml2026}

\usepackage{amsmath}
\usepackage{amssymb}
\usepackage{mathtools}
\usepackage{amsthm}

\usepackage{hyperref}
\usepackage{url}

\usepackage{graphicx}
\usepackage{epstopdf}
\usepackage{stmaryrd}
\usepackage{wrapfig}
\usepackage{nicefrac}
\usepackage{multirow}
\usepackage{mathtools} 
\usepackage{amsthm}
\usepackage{makecell} 
\usepackage{booktabs}
\usepackage[table,dvipsnames]{xcolor}

\usepackage[capitalize,noabbrev]{cleveref}

\theoremstyle{plain}
\newtheorem{theorem}{Theorem}[section]
\newtheorem{proposition}[theorem]{Proposition}
\newtheorem{lemma}[theorem]{Lemma}

\theoremstyle{definition}
\newtheorem{definition}[theorem]{Definition}

\theoremstyle{remark}

\DeclareMathOperator*{\argmin}{arg\,min}
\usepackage[textsize=tiny]{todonotes}

\begin{document}

\twocolumn[
  \icmltitle{Progressive Approximation in Deep Residual Networks: Theory and Validation}



  \icmlsetsymbol{equal}{*}

  \begin{icmlauthorlist}
    \icmlauthor{Wei Wang}{comp}
    \icmlauthor{Xiao-Yong Wei}{comp}
    \icmlauthor{Qing Li}{comp}
  \end{icmlauthorlist}


    \icmlaffiliation{comp}{Department of Computing, the Hong Kong Polytechnic University}
    \icmlcorrespondingauthor{Xiao-Yong Wei}{cs007.wei@polyu.edu.hk}


  \vskip 0.3in
]



\printAffiliationsAndNotice{}  

\begin{abstract}
The Universal Approximation Theorem (UAT) guarantees universal function approximation but does not explain how residual models distribute approximation across layers. We reframe residual networks as a layer-wise approximation process that builds an approximation trajectory from input to target, and prove the existence of progressive trajectories where error decreases monotonically with depth.
%
It reveals that residual networks can implement structured, step-by-step refinement rather than end-to-end (E2E) black-box mapping. 
Building on this, we propose Layer-wise Progressive Approximation (LPA), a theoretically grounded training principle that explicitly aligns each layer with its residual target to realize such trajectories. 
LPA is architecture-agnostic: we observe progressive behavior in residual FNNs, ResNets, and Transformers across tasks including complex surface fitting, image classification, and NLP with LLMs for generation and classification.
Crucially, this enables ``train once, use $N$ models": a single network yields useful predictions at every depth, supporting efficient shallow inference without retraining. 
Our work unifies approximation theory with practical deep learning, providing a new lens on representation learning and a flexible framework for multi-depth deployment.
The source code will be released unpon acceptance at \url{https://(open\_upon\_acceptance)}.
\end{abstract}

\begin{figure}[ht]
    \centering
    \includegraphics[width=0.45\textwidth]{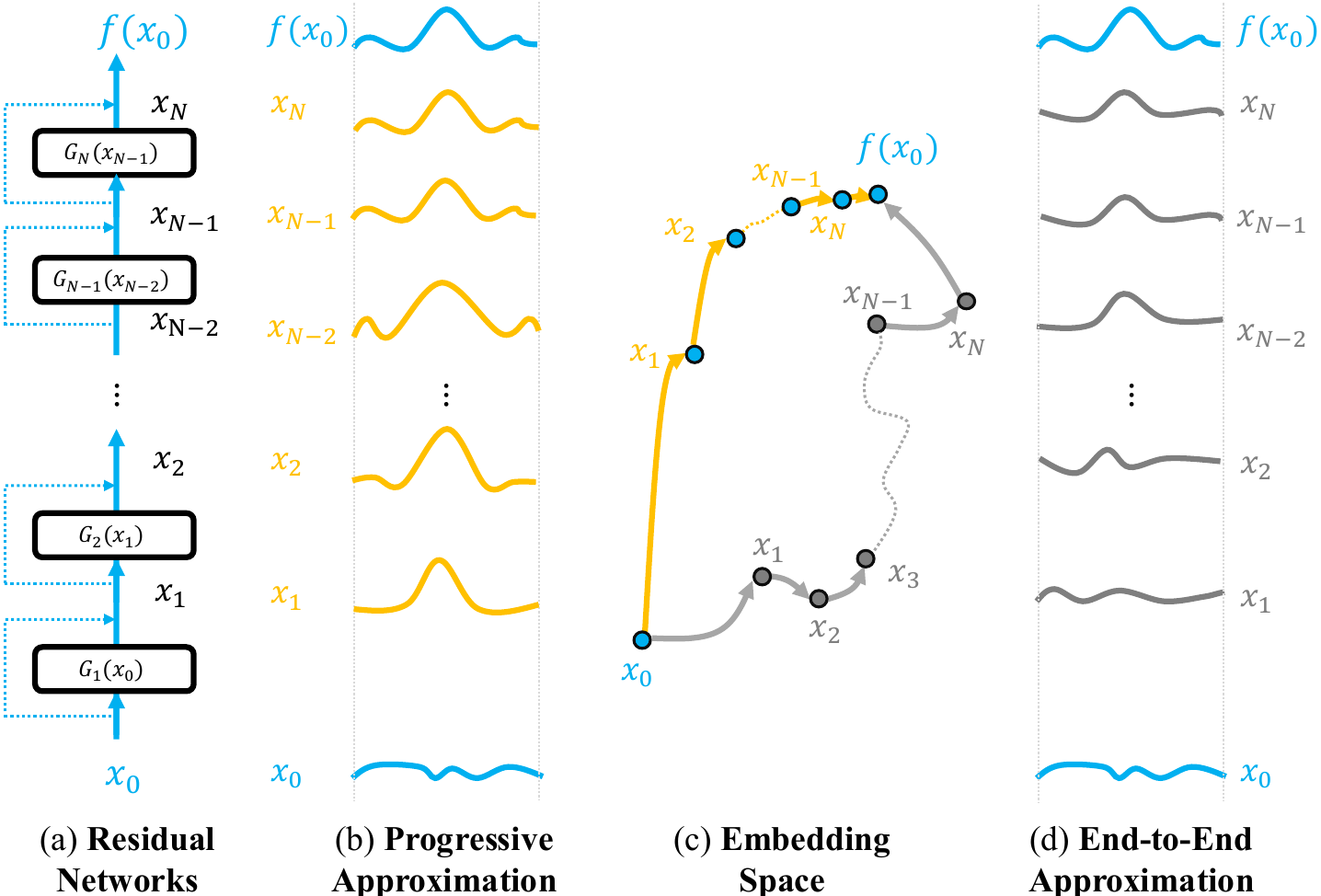}
    \caption{Progressive vs. E2E Approximation Trajectories: progressive approximation moves toward the target more effectively and yields useful intermediate layers, whereas intermediates in E2E training are less aligned with the target.}
    \label{fig:app_traj} 
\end{figure}

\section{Introduction}

The UAT~\cite{Cybenko1989ApproximationBS,Hornik1989MultilayerFN,Hornik1991ApproximationCO} guarantees that sufficiently large neural networks can approximate any continuous function. 
However, UAT is neither directly applicable to residual networks nor informative about how approximation develops during learning, especially in modern deep architectures where the input–output mapping is realized through many intermediate transformations.
In this work, we address a fundamental question left open by classical approximation theory: \textit{can the global approximation in deep residual networks be decomposed into a sequence of well-defined layer-wise approximation steps?}

We answer affirmatively. We show that training a deep network is equivalent to constructing an approximation trajectory $\{\mathbf{x}_0, \mathbf{x}_1, \dots, \mathbf{x}_N\}$ (see Figure~\ref{fig:app_traj}) that progressively moves from the input $\mathbf{x}_0$ toward the target $f(\mathbf{x}_0)$. 

We prove that there exists a class of progressive trajectories along which the approximation error $\|f(\mathbf{x}_0) - \mathbf{x}_i\|$ decreases monotonically with depth $i$. This reveals that deep residual networks do not merely act as black-box universal approximators; rather, they implement a structured, layer-by-layer refinement process grounded in approximation theory.

Building on this insight, we derive a principled training paradigm: Layer-wise Progressive Approximation (LPA), which explicitly encourages alignment with such trajectories. Unlike heuristic deep supervision, LPA is not an regularizer but a direct consequence of our theoretical characterization: each layer is guided to approximate the residual between the current state and the final target.

Our framework is architecture-agnostic. 
We validate progressive trajectories across full connected neural networks (FNNs), ResNets~\citep{He2015DeepRL}, and Transformers~\citep{Vaswani2017AttentionIA} on tasks spanning Simulation (function regression), Computer Vision (image classification), and NLP (language generation and classification with LLMs), indicating that this phenomenon is intrinsic to residual learning rather than specific to any single model or task.

In summary, our contributions are:
\begin{enumerate}
    \item A layer-wise decomposition of the global objective in residual networks, establishing a formal link between UAT and intermediate representations;
    
    \item A proof of the existence of monotonic progressive trajectories, providing theoretical grounding for intermediate supervision;
    
    \item A theoretical treatment of the LPA training principle with validation across architectures and tasks;
    
    \item A reframing from E2E training to progressive approximation, offering theoretical clarity and practical flexibility (e.g., ``train once, use $N$ models," demonstrated experimentally).
\end{enumerate}

\section{Related Work}

\textbf{Theoretical Analysis of Neural Approximation.}
The theoretical foundation of neural networks largely rests on the UAT~\citep{Cybenko1989ApproximationBS}, which establishes that sufficiently wide single-hidden-layer fully connected networks can approximate any continuous function on compact domains. Subsequent work extended UAT to multilayer architectures~\citep{Hornik1989MultilayerFN,Hornik1991ApproximationCO}.

With the rise of deep learning, a series of increasingly complex architectures have been developed including FNNs~\citep{Rumelhart1986LearningRB}, plain convolutional neural networks (CNNs)~\citep{LeCun1998GradientbasedLA,Krizhevsky2012ImageNetCW}, ResNet~\citep{He2015DeepRL}, Transformers~\citep{Vaswani2017AttentionIA}, BERT~\citep{devlin2019bert}, and Vision Transformers (ViTs)~\citep{dosovitskiy2020image}. While these models achieve remarkable performance, their growing architectural complexity has made rigorous analysis of their approximation behavior increasingly challenging.

Significant efforts have been devoted to establishing universal approximation guarantees for such modern architectures. For instance, \citet{Hardt2016IdentityMI,Lin2018ResNetWO} showed that FNNs retain universal approximation capability when equipped with residual connections; similarly, \citet{pmlr-v97-oono19a,He2021ApproximationPO} proved analogous results for residual CNNs. Transformers are also shown to be universal approximators~\citep{Yun2019AreTU,Kratsios2021UniversalAU}.

Fundamentally, prior work views the network as a monolithic universal approximator, concerned only with matching the final output to the target, and offers little insight into the layer-by-layer evolution of approximation. 
To our knowledge, this is the first study to leverage UAT to analyze approximation progressively at each layer. 
By decomposing the global objective into layer-wise targets, we expose the internal dynamics of approximation in residual networks.

\textbf{Algorithmic Approaches with Intermediate Supervision.}
Algorithmically, prior work on supervising intermediate layers falls into two broad categories. The first adds auxiliary losses, often unsupervised or self-supervised pretext objectives, to improve intermediate representations~\citep{Vincent2008ExtractingAC,Oord2018RepresentationLW,Pathak2016ContextEF,Hjelm2018LearningDR}. These objectives promote robust, discriminative features for downstream tasks but do not require intermediate layers to produce valid predictions.


The second category introduces additional modules (e.g., early-exit heads or auxiliary classifiers) to enable intermediate layers to output task-specific results~\citep{Bengio2006GreedyLT,Szegedy2015RethinkingTI,Bakhtiarnia2021MultiExitVT,Xu2023LGViTDE}. While effective for efficiency or adaptive inference, these approaches are largely heuristic: they lack a principled characterization of what each intermediate layer should learn, and offer no guarantee that the sequence of approximations improves monotonically toward the target.

We bridge this gap with the first formal theory of progressive approximation. Under a unified residual formulation, we show that the global objective decomposes into layerwise approximation problems and that a coherent trajectory with monotonically decreasing errors exists.

\section{Layer-Wise Progressive Approximation Trajectories in Residudal Networks}
\label{sec:LWPAT}
In this section, we give a justification for the existence of layer-wise progressive approximation trajectories.
\subsection{Residual Networks Unfolding and Error Trajectory}
\label{sec:rn_unfold}



%

Let $\{K_i\}_{i=1}^N \subseteq \mathbb{R}^d$ be compact subsets representing the input and intermediate representation spaces, and let $f: K_0 \to K_N$ be a continuous target function to be approximated. 
For an input $\mathbf{x}_0 \in K_0$, we denote by $\mathbf{x}_i \in K_i$ the output of the $i$-th layer of a deep residual network. The network follows the forward recurrence\footnote{For clarity of exposition, we present the analysis using FNNs. However, the framework is general and directly applies to ResNet and Transformer architectures, as both can be expressed within the similar residual formulation (see Appendix~\ref{apd:sec:unify_RNs} for details).}:
\begin{equation}
    \mathbf{x}_{i} = \mathbf{x}_{i-1} + \mathbf{W}'_{i} \sigma \big( \mathbf{W}_{i} \mathbf{x}_{i-1} + \mathbf{b}_{i} \big) + \mathbf{b}'_{i},
\end{equation}
where $\sigma$ denotes a pointwise nonlinearity (e.g., ReLU). Defining the residual mapping $G_i: K_{i-1} \to \mathbb{R}^d$ is a learnable residual block:
\begin{equation}
    G_i(\mathbf{x}_{i-1}) := \mathbf{W}'_{i} \sigma \big( \mathbf{W}_{i} \mathbf{x}_{i-1} + \mathbf{b}_{i} \big) + \mathbf{b}'_{i},
\end{equation}
which constitutes a FNN and shares the same mathematical form with UAT~\citep{Cybenko1989ApproximationBS}, the block simplifies to
\begin{equation}
    \mathbf{x}_{i} = \mathbf{x}_{i-1} + G_i(\mathbf{x}_{i-1}).
\end{equation}
Unfolding the recurrence over $N$ layers yields the cumulative representation (see Sections~\ref{sec:UnfoldtheRN} in the Appendix)
\begin{equation}
    g(\mathbf{x}_0) = \mathbf{x}_N = \mathbf{x}_0 + \sum_{i=1}^{N} G_i(\mathbf{x}_{i-1}).
    \label{eq:res_sum}
\end{equation}
For an E2E trained network, we denote the learned mappings by $G_i^*$. The behavior of intermediate representations can be analyzed via the following formal notion:

\begin{definition}[Error Trajectory]
Consider an $N$-layer residual network with forward pass $\mathbf{x}_i = \mathbf{x}_{i-1} + G_i^*(\mathbf{x}_{i-1})$ for $i = 1, \dots, N$. The approximation error at layer $i$ is defined as
\begin{equation}
    \varepsilon_i := \big\| f(\mathbf{x}_0) - \mathbf{x}_i \big\|, \quad i = 0, 1, \dots, N,
\end{equation}
and the sequence $\mathcal{T} = (\varepsilon_0, \varepsilon_1, \dots, \varepsilon_N)$ is called the error trajectory of the trained network on input $\mathbf{x}_0$.
\end{definition}

The objective of a residual $N$-layer network is to learn \(\{G_i^*\}_{1}^N\) approximates $f$:
\begin{equation}
    \{G_i^*\}_{1}^N=\argmin_{g}\|g(x_0)-f(x_0)\|.
\end{equation}
In the E2E paradigm, this global objective is applied only at the final layer, with earlier layers updated via the chain rule.

Since standard training minimizes only the final loss $\varepsilon_N$, the intermediate errors $\varepsilon_i$ ($i < N$) are not directly supervised. Consequently, the error trajectory $\mathcal{T}$ is generally non-monotonic and emerges as a byproduct of global optimization. For a more detailed derivation and discussion, see Sections~\ref{sec:Approximation_Trajectories} in the Appendix.

\subsection{Layer-Wise Approximation Perspective}
\label{subsec:layer_wise_approx}

Although residual networks are trained end-to-end (E2E) to minimize only the global error $\|f(\mathbf{x}_0) - \mathbf{x}_N\|$, their additive structure admits a natural layer-wise decomposition of the target function. This decomposition provides a useful analytical lens—even though E2E training does not enforce it explicitly.

\begin{proposition}[Implicit Layer-Wise Residual Targets]
    Let $\{G_i^*\}_{i=1}^N$ denote the residual mappings of an $N$-layer network after E2E training. Define the layer-$i$ residual target as
    \begin{equation}
        f_i(\mathbf{x}_0) := f(\mathbf{x}_0) - \mathbf{x}_0 - \sum_{j=1}^{i-1} G_j^*(\mathbf{x}_{j-1}),
        \label{eq:target_layer_i}
    \end{equation}
    where $\mathbf{x}_j = \mathbf{x}_{j-1} + G_j^*(\mathbf{x}_{j-1})$ for $j < i$. 
    Then, by construction, the ideal update at layer $i$ would be $G_i^{\text{ideal}}(\mathbf{x}_{i-1}) = f_i(\mathbf{x}_0)$.
    \label{prop:layer_wise_targets}
\end{proposition}

Proposition~\ref{prop:layer_wise_targets} does not claim that E2E training actively steers $G_i^*$ toward $f_i(\mathbf{x}_0)$. Rather, it highlights that, given any fixed set of preceding layers, we could define the residual target $f_i$, and the role of $G_i$ can be interpreted as approximating this target.

To assess whether such an approximation is theoretically feasible, consider the effective input domain to layer $i$:
\[
    K_{i-1} := \big\{ \mathbf{x}_{i-1}(\mathbf{x}_0) \mid \mathbf{x}_0 \in K_0 \big\}.
\]
Assume $K_0 \subset \mathbb{R}^d$ is compact and all $G_j^*$ are continuous; then $K_{i-1}$ is also compact. Since $f_i(\mathbf{x}_0)$ is a continuous function of $\mathbf{x}_0$, and $\mathbf{x}_{i-1}$ is a continuous function of $\mathbf{x}_0$, the mapping from $\mathbf{x}_{i-1}$ to $f_i(\mathbf{x}_0)$ is continuous on $K_{i-1}$. By the UAT~\citep{Cybenko1989ApproximationBS}, for any $\varepsilon_i > 0$, there exists a FNN $G_i$ satisfying
\begin{equation}
    \sup_{\mathbf{x}_{i-1} \in K_{i-1}} \big\| G_i(\mathbf{x}_{i-1}) - f_i(\mathbf{x}_0) \big\| < \varepsilon_i.
    \label{eq:uat_layer_i}
\end{equation}
Thus, from an approximation-theoretic standpoint, each residual block is capable of realizing its implicit target, provided sufficient capacity.

However, because standard E2E training optimizes only the final output, intermediate layers receive no explicit signal to match $f_i$. Consequently, the actual error trajectory $\{\varepsilon_i\}_{i=1}^N$ (defined in Section~\ref{sec:rn_unfold}) may be non-monotonic in practice.

\subsection{Progressive Approximation Trajectories}
\label{subsec:prog_traj}

Proposition~\ref{prop:layer_wise_targets} shows that, in a globally optimal residual network, each layer $i$ implicitly targets the residual function
\[
    f_i(\mathbf{x}_0) = f(\mathbf{x}_0) - \mathbf{x}_{i-1},
\]
which represents the part of the target not yet captured by the first $i-1$ layers. This raises a natural question: does there exist a choice of residual mappings $\{G_i\}_{i=1}^N$ such that the approximation error never increases with depth? We now show that such a trajectory indeed exists.

\begin{theorem}[Existence of Progressive Trajectory]
    There exists a sequence of residual mappings $\{G_i\}_{i=1}^N$ such that the corresponding uniform errors
    \[
        \varepsilon_i := \sup_{\mathbf{x}_0 \in K_0} \big\| f(\mathbf{x}_0) - \mathbf{x}_i \big\|
    \]
    satisfy
    \begin{equation}
        \varepsilon_1 \geq \varepsilon_2 \geq \cdots \geq \varepsilon_N \geq 0.
    \end{equation}
    \label{th:prog_traj}
\end{theorem}
\begin{proof}
At layer 1, the residual to approximate is $f_1(\mathbf{x}_0) = f(\mathbf{x}_0) - \mathbf{x}_0$. By the UAT, there exists a FFN $G_1^*$ such that the output $\mathbf{x}_1 = \mathbf{x}_0 + G_1^*(\mathbf{x}_0)$ satisfies
\[
    \varepsilon_1 = \sup_{\mathbf{x}_0 \in K_0} \|f(\mathbf{x}_0) - \mathbf{x}_1\| < \infty.
\]
Assume that after $i-1$ layers, we have obtained $\mathbf{x}_{i-1}$ with error $\varepsilon_{i-1}$. The remaining residual is $f_i(\mathbf{x}_0) = f(\mathbf{x}_0) - \mathbf{x}_{i-1}$. Since $f_i$ is a continuous function of $\mathbf{x}_0$, and the input to layer $i$ is $\mathbf{x}_{i-1}$ (which is determined by $\mathbf{x}_0$), the universal approximation capability implies that there exists a network $G_i$ capable of approximating this residual from $\mathbf{x}_{i-1}$. In particular, we can choose $G_i^*$ such that the new error satisfies
\[
    \varepsilon_i = \sup_{\mathbf{x}_0 \in K_0} \|f_i(\mathbf{x}_0) - G_i^*(\mathbf{x}_{i-1})\| \leq \varepsilon_{i-1}.
\]

This is always possible because the best achievable error at layer $i$ cannot exceed the current residual magnitude $\varepsilon_{i-1}$; otherwise, one could simply set $G_i \equiv 0$ and retain error $\varepsilon_{i-1}$. Hence, by selecting $G_i$ appropriately (e.g., better than zero), we ensure $\varepsilon_i \leq \varepsilon_{i-1}$. By induction, there exists a sequence $\{G_i\}_{i=1}^N$ yielding a non-increasing error trajectory:
\[
    \varepsilon_1 \geq \varepsilon_2 \geq \cdots \geq \varepsilon_N \geq 0.
\]
This completes the proof.
\end{proof}

Thus, progressive approximation trajectories are theoretically realizable: the compositional structure of residual networks, combined with the universal approximation property, guarantees the existence of a depth-wise refinement path along which the worst-case error does not increase.

\section{Learning for Progressive Approximation}

To realize the progressive approximation trajectory established in theory, we design a practical training paradigm that aligns intermediate layers with the global objective while respecting architectural constraints. The key challenge is that intermediate representations $\mathbf{x}_i$ typically live in a hidden space of dimension $d$, which may differ from the output dimension $d_y$ of the target function $f: \mathbb{R}^d \to \mathbb{R}^{d_y}$. This dimensional mismatch prevents direct application of the loss $\mathcal{L}(\mathbf{x}_i, f(\mathbf{x}_0))$ at intermediate layers. 

Our solution is to introduce a shared, learnable linear readout that projects all layer outputs into the target space. This leads to the following design principles:

\begin{itemize}
    \item \textbf{Global supervision at all depths:} Every layer receives a learning signal derived from the global target $f(\mathbf{x}_0)$, enabling error feedback at all depths.
    
    \item \textbf{Consistent output-space alignment:} All layers use the same readout mapping for predictions, ensuring their objectives are defined in a common space and remain coupled through the shared parameters.
    
\end{itemize}
We now describe how these principles are instantiated in the LPA framework.

\begin{algorithm}[htp]
\caption{LPA Training}
\label{alg:LPA}
\begin{algorithmic}[1]
\STATE Preprocessing layer: $G_0()$; Residual network of $N$ layers $\{G_i()\}_1^N$ and weights $\{\theta_i\}_1^N$; Linear transformation for prediction: $\mathbf{W}_{N+1}$; Training data: $\{(\mathbf{x}, \mathbf{y})\}$; Loss function: $\ell$.

\FOR{epoch $= 1$ \textbf{to} $N_{\text{epochs}}$}
        \FOR{each batch $(\mathbf{x}_b, \mathbf{y}_b)$}
            \STATE $\mathbf{x}_0 \leftarrow G_0(\mathbf{x}_b)$
            \FOR{$i = 1$ \textbf{to} $L$} 
                \STATE $\mathbf{x}_i \leftarrow \mathbf{x}_{i-1} + G_i(\mathbf{x}_{i-1})$ 
                \STATE $\hat{\mathbf{y}}_i \leftarrow \mathbf{W}_{N+1}\mathbf{x}_i$ 
                \STATE $\mathcal{L}_i \leftarrow \ell(\hat{\mathbf{y}}_i, \mathbf{y}_b)$
            \ENDFOR
        \STATE Backpropagate $\mathcal{L}=\Sigma_{i=1}^N\lambda_i\mathcal{L}_i$ and update $\{\theta_1, ..., \theta_N, \mathbf{W}_{N+1}.\}$
        \ENDFOR
\ENDFOR
\end{algorithmic}
\end{algorithm}

\subsection{Synchronized Layer Objectives via Shared Readout}

Let $\mathbf{x}_i \in \mathbb{R}^{d}$ denote the output of layer $i$, and let $f(\mathbf{x}_0) \in \mathbb{R}^{d_y}$ be the target. Since $d \neq d_y$ in general, we introduce a learnable linear readout matrix $\mathbf{W}_{N+1} \in \mathbb{R}^{d_y \times d}$ and share it across all layers. The prediction at depth $i$ is then:
\[
    \hat{f}_i(\mathbf{x}_0) = \mathbf{W}_{N+1} \mathbf{x}_i.
\]
Note that $\mathbf{W}_{N+1}$ is jointly optimized with the residual blocks $\{G_i\}_{i=1}^N$ during training. According to Eq.\ref{eq:res_sum}, the  output-space of a residual network prediction becomes a cumulative sum:
\begin{equation}
\begin{aligned}
    \hat{f}_i(\mathbf{x}_0) &= \mathbf{W}_{N+1}\big(\mathbf{x}_0 + \sum_{j=1}^i  G_j(\mathbf{x}_{j-1})\big)\\
    &= \mathbf{W}_{N+1} \mathbf{x}_0 + \sum_{j=1}^i \underbrace{\mathbf{W}_{N+1} G_j(\mathbf{x}_{j-1})}_{=: H_i(\mathbf{x}_{j-1})}.
    \label{eq:layer_rep_shh}    
\end{aligned}
\end{equation}
Each term $H_j: \mathbb{R}^{d} \to \mathbb{R}^{d_y}$ can be viewed as an output-space correction contributed by block $j$.

\begin{table*}[h]
\centering
\caption{Performance comparison of the proposed LPA and E2E optimization on complex surface fitting tasks. The best results are in bold.}
\label{tab:LPAT-ori}
\resizebox{\textwidth}{!}{
\begin{tabular}{l|c|c|c|c|c|c|c|c|c|c|c|c|c|c|c|c} 
\toprule
\multicolumn{1}{l}{\makecell[l]{\textbf{Datasets}}} 
& \multicolumn{2}{c}{\makecell[c]{\textbf{Anisotropic}\\\textbf{Radial Oscillations}}}
& \multicolumn{2}{c}{\makecell[c]{\textbf{Damped}\\\textbf{Concentric Ripples}}}
& \multicolumn{2}{c}{\makecell[c]{\textbf{Exponential Saddle}}}
& \multicolumn{2}{c}{\makecell[c]{\textbf{Quadratically}\\\textbf{Warped Oscillations}}}
& \multicolumn{2}{c}{\makecell[c]{\textbf{Standing}\\\textbf{Wave Grid}}}
& \multicolumn{2}{c}{\makecell[c]{\textbf{Asymmetric Power-}\\\textbf{Modulated Wave}}}
& \multicolumn{2}{c}{\makecell[c]{\textbf{Localized High-}\\\textbf{Frequency Stripes}}}
& \multicolumn{2}{c}{\makecell[c]{\textbf{Logarithmic}\\\textbf{Central Well}}} \\ 
\cmidrule(r){2-3}\cmidrule(rl){4-5}\cmidrule(rl){6-7}\cmidrule(rl){8-9}\cmidrule(rl){10-11}\cmidrule(rl){12-13}\cmidrule(rl){14-15}\cmidrule(l){16-17}
\multicolumn{1}{l}{\textbf{Methods}} 
& \textbf{MSE $\downarrow$} & \textbf{MAE $\downarrow$}
& \textbf{MSE $\downarrow$} & \textbf{MAE $\downarrow$}
& \textbf{MSE $\downarrow$} & \textbf{MAE $\downarrow$}
& \textbf{MSE $\downarrow$} & \textbf{MAE $\downarrow$}
& \textbf{MSE $\downarrow$} & \textbf{MAE $\downarrow$}
& \textbf{MSE $\downarrow$} & \textbf{MAE $\downarrow$}
& \textbf{MSE $\downarrow$} & \textbf{MAE $\downarrow$}
& \textbf{MSE $\downarrow$} & \textbf{MAE $\downarrow$} \\
\midrule

\textbf{E2E}      
& 0.0081 & 0.0458
& 0.1431 & 0.2535
& 0.0048 & 0.0428
& 0.0159 & 0.0679
& 0.0062 & 0.0441
& 0.0071 & 0.0474
& 0.0010 & 0.0238
& 0.0109 & 0.0654 \\

\textbf{LPA}        
& \textbf{0.0016} & \textbf{0.0292}
& \textbf{0.0109} & \textbf{0.0752}
& \textbf{0.0005} & \textbf{0.0173}
& \textbf{0.0077} & \textbf{0.0599}
& \textbf{0.0011} & \textbf{0.0239}
& \textbf{0.0019} & \textbf{0.0322}
& \textbf{0.0009} & \textbf{0.0207}
& \textbf{0.0087} & \textbf{0.0649} \\
\bottomrule
\end{tabular}}
\end{table*}

\begin{figure*}[ht]
    \centering
    \includegraphics[width=0.98\textwidth]{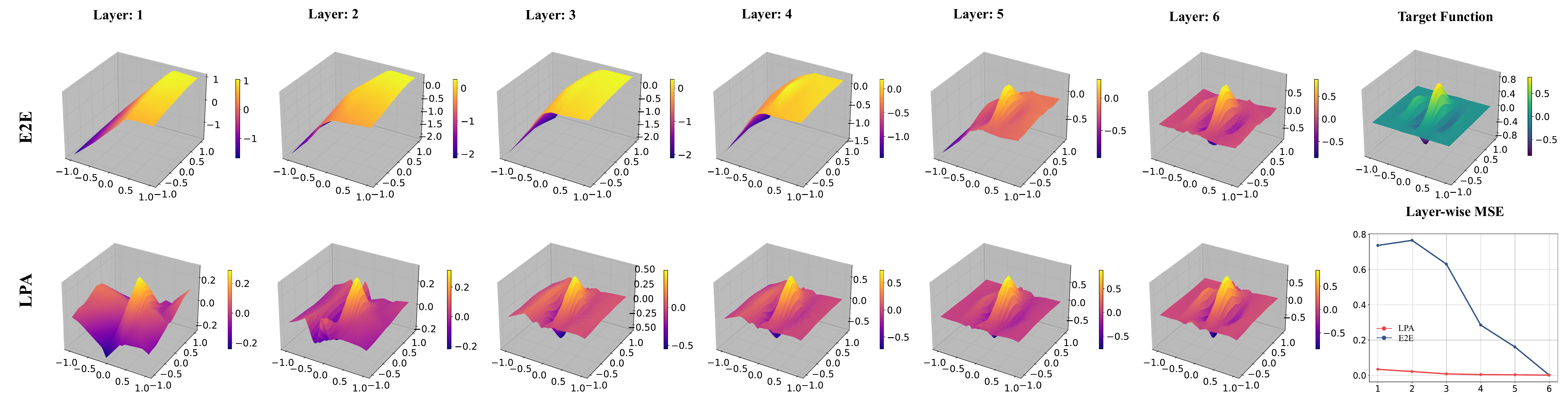}
    \caption{Layer-wise approximation results of LPA and E2E optimizations. The target function is $f(x, y) = e^{-5(x^2 + y^2)} \sin(10x)$ with Localized high-frequency stripe. LPA exhibits clear progression, reaching near-perfect approximation by layer 3.}
    \label{fig:LHFS} 
\end{figure*}

The LPA loss aggregates supervision signals from all layers:
\begin{equation}
    \mathcal{L}_{\text{LPA}} = \sum_{i=1}^N \lambda_i \, \ell\big( \hat{f}_i(\mathbf{x}_0),\, f(\mathbf{x}_0) \big),
    \label{eq:lpa_loss}
\end{equation}
where $\lambda_i \geq 0$ are weighting coefficients. Because all $\hat{f}_i$ are computed using the same $\mathbf{W}_{N+1}$, the layer-wise losses are inherently coupled, so the updating of $\mathbf{W}_{N+1}$ affects all terms simultaneously, enforcing consistent output-space alignment. The implementation as shown in Algorithm \ref{alg:LPA}. By supervising earlier layers, we constrain the residual passed to deeper blocks, facilitating the non-increasing error behavior suggested by Theorem~\ref{th:prog_traj}.

\begin{figure*}[ht]
    \centering
    \includegraphics[width=0.95\textwidth]{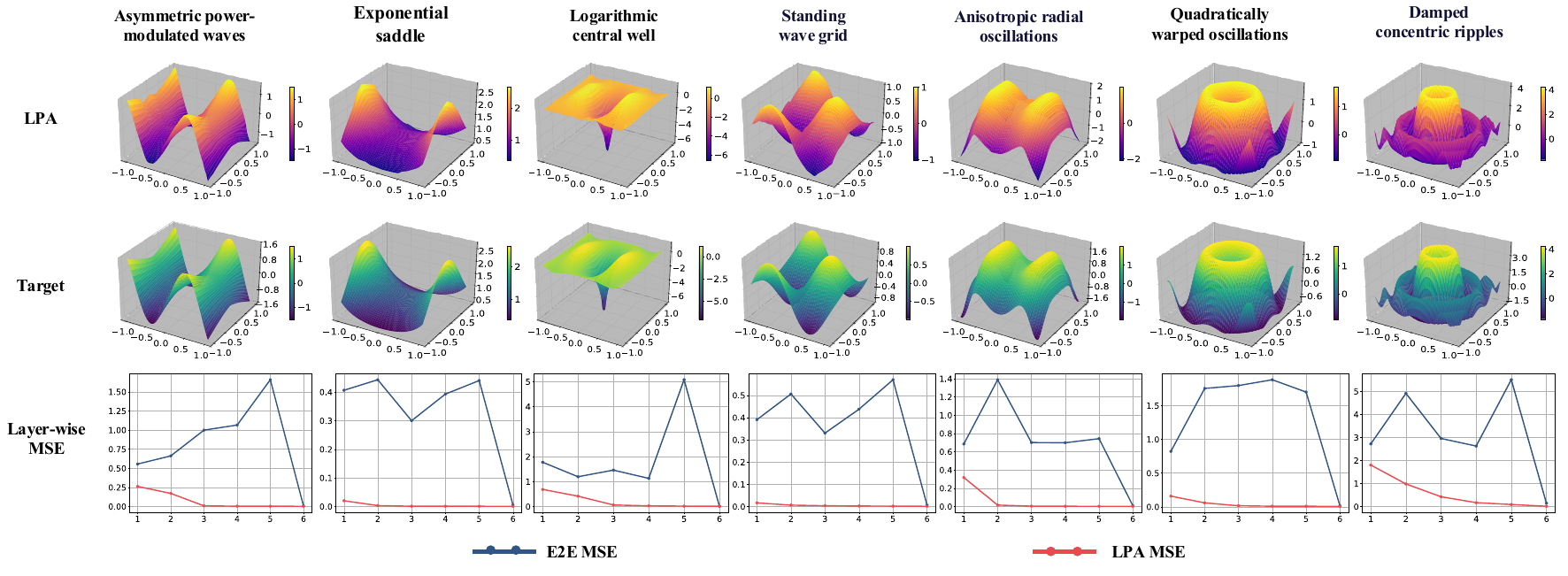}
    \caption{Additional complex surface fitting results with LPA, and a layerwise performance comparison between LPA and E2E.}
    \label{fig:csf} 
\end{figure*}

\subsection{Progression Preservation with Shared Readout}
\label{sec:shared_readout}

A natural question arises: does using a shared final linear layer $\mathbf{W}_{N+1}$, applied to all intermediate representations, preserve the layer-wise approximation and progressive refinement properties? This design is not a trick but aligns with the intrinsic structure of residual networks. Recall that the network output
\begin{equation}
    \hat{f}_{N+1}(\mathbf{x}_0) = \mathbf{W}_{N+1} \mathbf{x}_N,
\end{equation}
and by unrolling the residual recurrence $\mathbf{x}_i = \mathbf{x}_{i-1} + G_i(\mathbf{x}_{i-1})$, we obtain (as in Eq.~\ref{eq:res_sum})
\begin{equation}
    \hat{f}_{N+1}(\mathbf{x}_0) = \mathbf{W}_{N+1} \mathbf{x}_0 + \sum_{j=1}^N \underbrace{\mathbf{W}_{N+1} G_j(\mathbf{x}_{j-1})}_{=: H_j(\mathbf{x}_{j-1})}.
\end{equation}
This decomposition reveals that the contribution of each residual block $G_j$ to the final prediction is necessarily mediated through the same linear map $\mathbf{W}_{N+1}$. Therefore, evaluating intermediate layers via the shared readout
\begin{equation}
    \hat{f}_i(\mathbf{x}_0) := \mathbf{W}_{N+1} \mathbf{x}_i = \mathbf{W}_{N+1} \mathbf{x}_0 + \sum_{j=1}^i H_j(\mathbf{x}_{j-1})
    \label{eq:sro}
\end{equation}
is not only practical but also theoretically faithful to the network’s native computation. Under this formulation, each $\hat{f}_i$ remains a continuous function on the compact input domain $K_0$, and thus satisfies the conditions of the UAT. A layer-wise UAT argument, analogous to Proposition~\ref{prop:layer_wise_targets}, guarantees that for any target decomposition $\{f_i\}_{i=1}^N$, there exist residual blocks $\{G_j^*\}$ such that the induced functions $\{\hat{f}_i\}$ approximate the desired trajectory.

Moreover, the approximation error at layer $i$ can be bounded as
\begin{equation}
    \sup_{\mathbf{x}_0 \in K_0} \big\| f(\mathbf{x}_0) - \hat{f}_i(\mathbf{x}_0) \big\| < \varepsilon_i,
\end{equation}
with $\varepsilon_1 \geq \varepsilon_2 \geq \cdots \geq \varepsilon_N$, establishing a monotonic progressive trajectory as in Theorem~\ref{th:prog_traj}.

Hence, the shared readout preserves both the expressivity and the progressive refinement property. A full constructive proof with explicit error propagation bounds is provided in Appendix~\ref{sec:SharedReadout}. Besides, the layer-wise representation under intermediate dimension changes is discussed in Appendix~\ref{sec:dim_change_proof}.

\section{Experiments}
To assess progressive approximation, we test three settings:
\begin{itemize}
    \item Complex surface fitting: residual FNNs regressing a diverse set of 8 scientific functions to examine the performance and progressive behavior of intermediate layers.
    \item Image classification: ViTs and ResNets on 4 standard datasets to assess architectural compatibility, with large-scale ImageNet to test scalability.
    \item Large language models: QWen on text classification and generation benchmarks to evaluate scalability and task compatibility.
\end{itemize}
Across simulations and real-world vision and language tasks, we observe consistent progressive approximation across architectures and tasks. Full implementation details are provided in Appendix~\ref{sec:exp_details}. 

\subsection{Experiments for Complex Surface Fitting}
\label{sec:Curve_Fitting}
%
We generate 8 simulation datasets using functions as follows: Asymmetric power-modulated waves $f(x, y) = \big( |x|^{0.7} + |y|^{1.3} \big) \sin(4x)$, Localized high-frequency stripes $ f(x, y) = e^{-5(x^2 + y^2)} \sin(10x)$, Exponential saddle $ f(x, y) = e^{x^2 - y^2} $, Logarithmic central well $ f(x, y) = \log(x^2 + y^2 + 10^{-5}) \cos(5x) $, Standing wave grid $ f(x, y) = \sin(3x) \cos(3y) $, Anisotropic radial oscillations $ f(x, y) = \dfrac{\sin(4x^2 + y^2)}{\sqrt{x^2 + y^2 + 0.001}} $, Quadratically warped oscillations: $ f(x, y) = \sin(5x^2) + \cos(3y^2) $, and Damped concentric ripples: $ f(x, y) = \dfrac{\sin\!\big(10(x^2 + y^2)\big)}{x^2 + y^2 + 0.1} $. Full descriptions are provided in Appendix~\ref{sec:Curve_Fitting_data}.
There are 20,000 input-output pairs for each target function (10,000 pairs for training and 10,000 pairs for testing).

The target functions are complex, non-linear surfaces defined over $\mathbb{R}^2$. All experiments employ a standardized 6-layer residual FNNs, where each layer consists of a 30-width MLP with ReLU activation and residual connections. We compare two training paradigms: E2E optimization and our proposed LPA.

As shown in Table~\ref{tab:LPAT-ori}, LPA achieves lower final MSE and MAE across all benchmarks compared to E2E, with improvements ranging from $1.1\times$ to $13.1\times$ in MSE. More importantly, the aim here is not merely superior final accuracy but to confirm that the network follows the progressive approximation trajectory predicted by theory.

Figure~\ref{fig:LHFS} shows a layerwise comparison for $f(x, y) = e^{-5(x^2 + y^2)} \sin(10x)$. With E2E training, mid-layer flexibility lacks early guidance and often fails to refine the approximation. 
In contrast, LPA allows early layers to capture coarse structure and later layers to progressively refine it, yielding near perfect alignment at the final layer. 
This matches the progressive trajectory in Theorem~\ref{th:prog_traj}.


Figure~\ref{fig:csf} further illustrates the final approximated surfaces under LPA paradigms. LPA consistently produces smoother, faithful reconstructions of the target functions. Moreover, the evolution of layer-wise MSE reveals that LPA maintains a monotonically decreasing trend, whereas E2E exhibits fluctuating errors, sometimes increasing mid-training, indicating unstable learning dynamics.

These results confirm that LPA not only improves final performance in this controlled setting but also enforces a theoretically sound approximation process: each layer contributes and cumulatively toward the global objective.

\begin{table}[htbp]
\centering
\caption{Performance comparison of LPA and E2E with ViT on three benchmarks of CIFAR-10 (C10), CIFAR-100 (C100), and Fashion-MNIST (FM). Best scores are in bold.}
\label{tab:im_cls_Vit}
\resizebox{0.5\textwidth}{!}{
\begin{tabular}{ll|ccc|ccc}
\noalign{\hrule height 0.8pt}
\multicolumn{2}{l|}{\textbf{Models}} & \multicolumn{6}{c}{\textbf{ViT (Transformers)}} \\
\hline
\multicolumn{2}{l|}{\multirow{2}{*}{\textbf{Training Paradigms}}} 
& \multicolumn{3}{c|}{\textbf{\#Layers = 12}} & \multicolumn{3}{c}{\textbf{\#Layers = 24}} \\
\cline{3-8}
\multicolumn{2}{l|}{} 
& \textbf{C10} & \textbf{C100} & \textbf{FM} & \textbf{C10} & \textbf{C100} & \textbf{FM} \\
\hline
\multirow{2}{*}{\textbf{E2E}} 
& \textbf{Acc.} $\uparrow$   & 84.85 & \textbf{60.83} & 92.90 & 86.59 & 60.65 & \textbf{93.40} \\
& \textbf{Loss} $\downarrow$ & 0.47  & 1.53  & 0.19  & 0.46  & 1.72  & 0.19  \\
\hline
\multirow{2}{*}{\textbf{LPA}} 
& \textbf{Acc.} $\uparrow$   & \textbf{85.16} & 60.0 & \textbf{93.00} & \textbf{86.95} & \textbf{60.69} & 93.39 \\
& \textbf{Loss} $\downarrow$ & 0.47  & 1.62  & 0.19  & 0.47  & 1.79  & 0.20  \\
\noalign{\hrule height 0.8pt}
\end{tabular}}
\end{table}

\begin{table}[htp]
\centering
\caption{Performance comparison of LPA and E2E with ResNet on three benchmarks of CIFAR-10 (C10), CIFAR-100 (C100), and Fashion-MNIST (FM). Best scores are in bold.}
\label{tab:im_cls_CNN}
\resizebox{0.5\textwidth}{!}{
\begin{tabular}{ll|c|c|c|c|c|c} 
\noalign{\hrule height.8pt}
\multicolumn{2}{l|}{\textbf{Models}} & \multicolumn{6}{c}{\textbf{\textbf{ResNet (Residual CNN)}}} \\ 
\hline
\multicolumn{2}{l|}{\multirow{2}{*}{\textbf{Training Paradigms}}} & \multicolumn{3}{c|}{\textbf{\#Layers = 18}} & \multicolumn{3}{c}{\textbf{\#Layers = 34}} \\ 
\cline{3-8}
\multicolumn{2}{l|}{} & \textbf{C10} & \textbf{C100} & \textbf{FM} & \textbf{C10} & \textbf{C100} & \textbf{FM} \\ 
\hline
\multirow{2}{*}{\textbf{E2E}} & \textbf{Acc.} $\uparrow$ & \textbf{85.70} & 61.83 & 91.54 & 87.57 & \textbf{64.34} & \textbf{92.52} \\
& \textbf{Loss} $\downarrow$ & 0.42 & 1.35 & 0.23 & 0.47 & 1.27 & 0.20 \\\hline
\multirow{2}{*}{\textbf{LPA}} & \textbf{Acc.} $\uparrow$ & 85.44 & \textbf{62.34} & \textbf{92.54} & \textbf{87.81} & 63.95 & 92.03 \\
& \textbf{Loss} $\downarrow$ & 0.30 & 1.22 & 0.20 & 0.24 & 1.08 & 0.17 \\
\noalign{\hrule height.8pt}
\end{tabular}}
\end{table}

\begin{figure}[ht]
   \centering
    \includegraphics[width=0.48\textwidth]{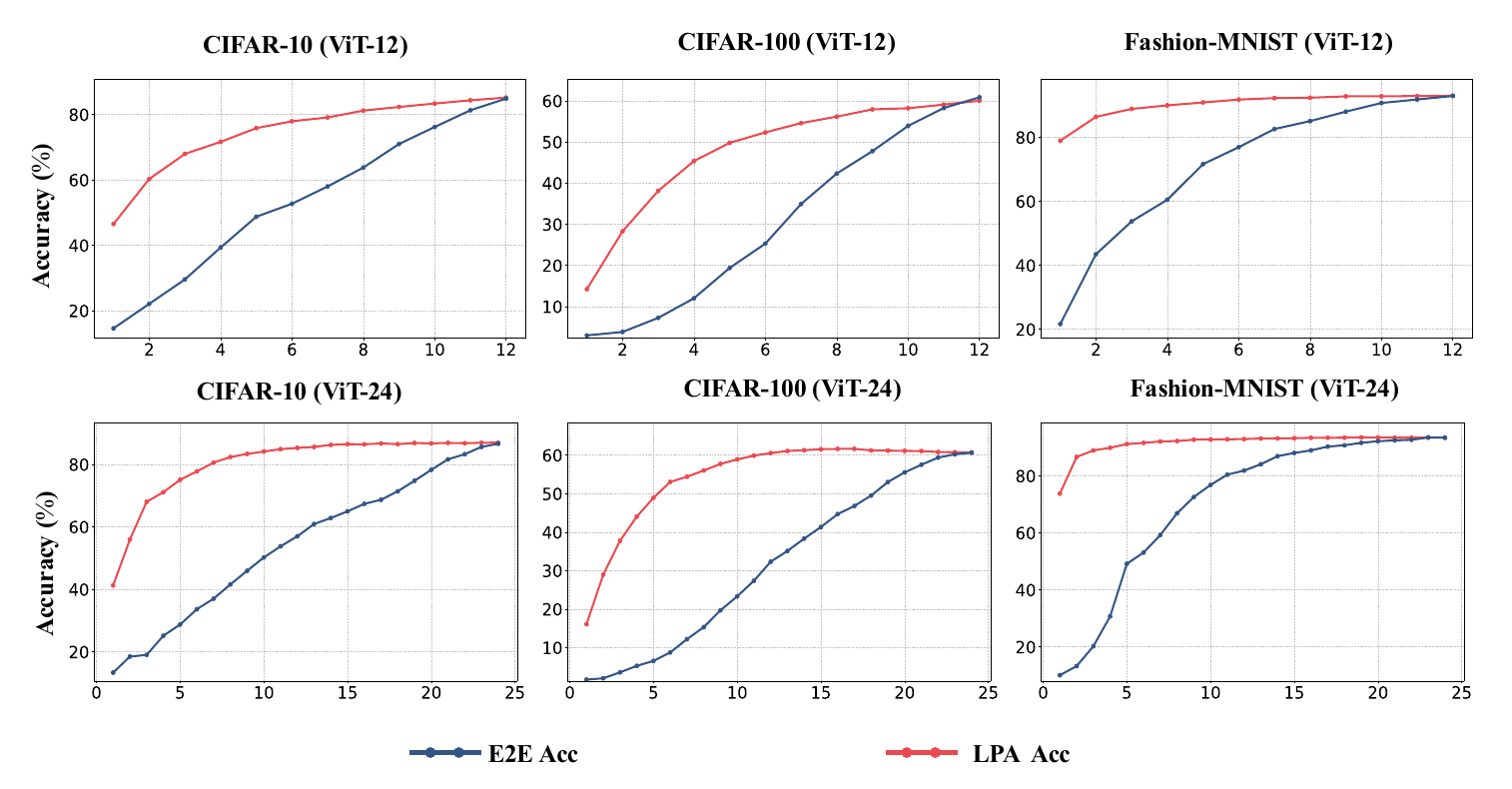} 
   \caption{Layer-wise performance comparison with ViT. Convergence at early layers appears across all three datasets, with strong approximation by layer 5 of 24 on Fashion-MNIST.} 
   \label{fig:vit_cur} 
\end{figure}

\begin{figure}[ht]
   \centering
    \includegraphics[width=0.48\textwidth]{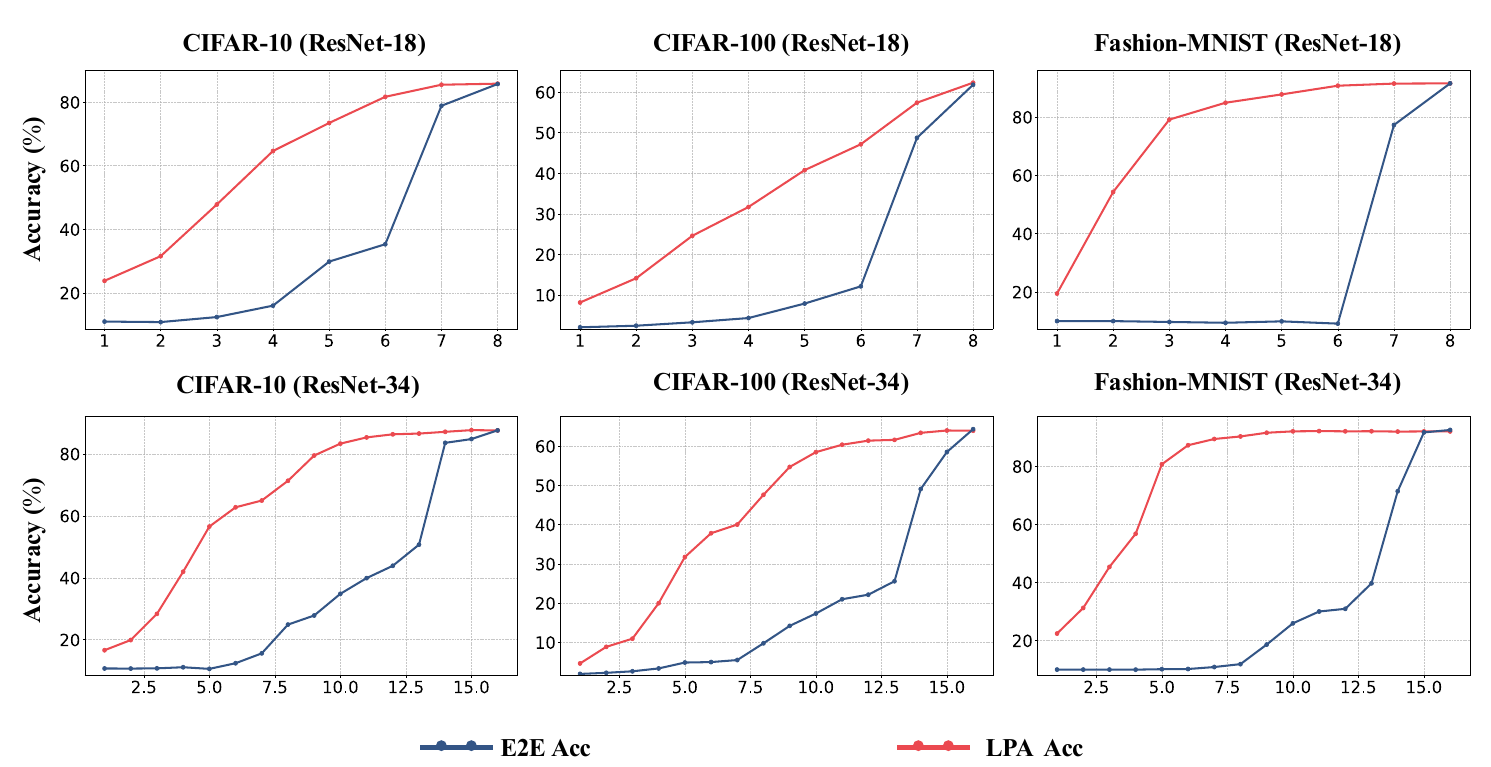} 
   \caption{Layer-wise performance comparison with ResNet. Note: ResNet-18 and ResNet-34 contain 8 and 16 residual blocks, respectively; ``layer" here refers to residual blocks.} 
   \label{fig:resnet_cur} 
\end{figure}

\subsection{Experiments for Image classification}
\label{sec:im_cls}
We evaluate LPA on CIFAR-10, CIFAR-100, and Fashion-MNIST using both ResNet-18 and ResNet-34 (approximately 11.2M and 21.3M parameters) and Vision Transformers with 12 and 24 layers. Each ViT layer uses 3 attention heads of size 192 and an MLP hidden size of 768, yielding models of about 5.5M and 10.9M parameters. ViTs are trained for 300 epochs and ResNets for 50 epochs. 


In Tables~\ref{tab:im_cls_Vit} and \ref{tab:im_cls_CNN},
LPA matches E2E in final accuracy across all settings, indicating that progressive supervision does not degrade performance. 
More importantly, our goal is to verify that LPA induces a progressive approximation trajectory in real datasets, as predicted by theory.


We evaluate layer-wise accuracy at inference by feeding intermediate features through the shared readout. 
As shown in Figures~\ref{fig:vit_cur} and \ref{fig:resnet_cur}, the accuracy generally increases with depth for both methods, but LPA consistently delivers much higher performance in the early layers. (We use ResNet-18 and ResNet-34, which contain 18 and 34 computational layers, respectively. However, they consist of only 8 and 16 residual blocks. Throughout this work, ``layer-wise" refers to these residual blocks. Thus, the curves in Figure~\ref{fig:resnet_cur} correspond to 8 and 16 layers, respectively.). 
In the 24-layer ViT, LPA attains near-optimal accuracy at intermediate depths, around layer 15 on CIFAR-10/100 and layer 10 on Fashion-MNIST, and then refines smoothly. 
In contrast, E2E leaves early layers under-optimized, with peak accuracy only at the final output.

This property has practical implications: users can deploy shallower subnetworks of an LPA-trained model to meet computational or latency constraints, without requiring distillation or retraining. The model is inherently depth-adaptive and has implemented ``train once, use $N$ models".

\paragraph{Scalability test.} We train ViT-12 and ViT-24 from scratch on ImageNet-1k \citep{Deng2009ImageNetAL} with both LPA and E2E. As shown in Table~\ref{tab:vit_imagenet_from_scratch}, LPA matches final accuracy while preserving progressive behavior. For example, ViT-24 with LPA surpasses 80\% top-1 accuracy by layer 14, indicating that strong representations emerge well before full depth.

Finally, we examine whether LPA can fine-tune pre-trained models. In Table~\ref{tab:vit_layerwise_finetune30}, we fine-tune pre-trained ViT-Small and ViT-Base for 30 epochs on Imagenet using LPA. Again, LPA yields significantly better intermediate-layer performance: strong accuracy is achieved early in the network, whereas E2E remains heavily reliant on the final layer. This confirms that the progressive approximation property is not limited to training from scratch but is also capable of finetuning.

\begin{figure}[ht]
   \centering
   \includegraphics[width=0.48\textwidth]{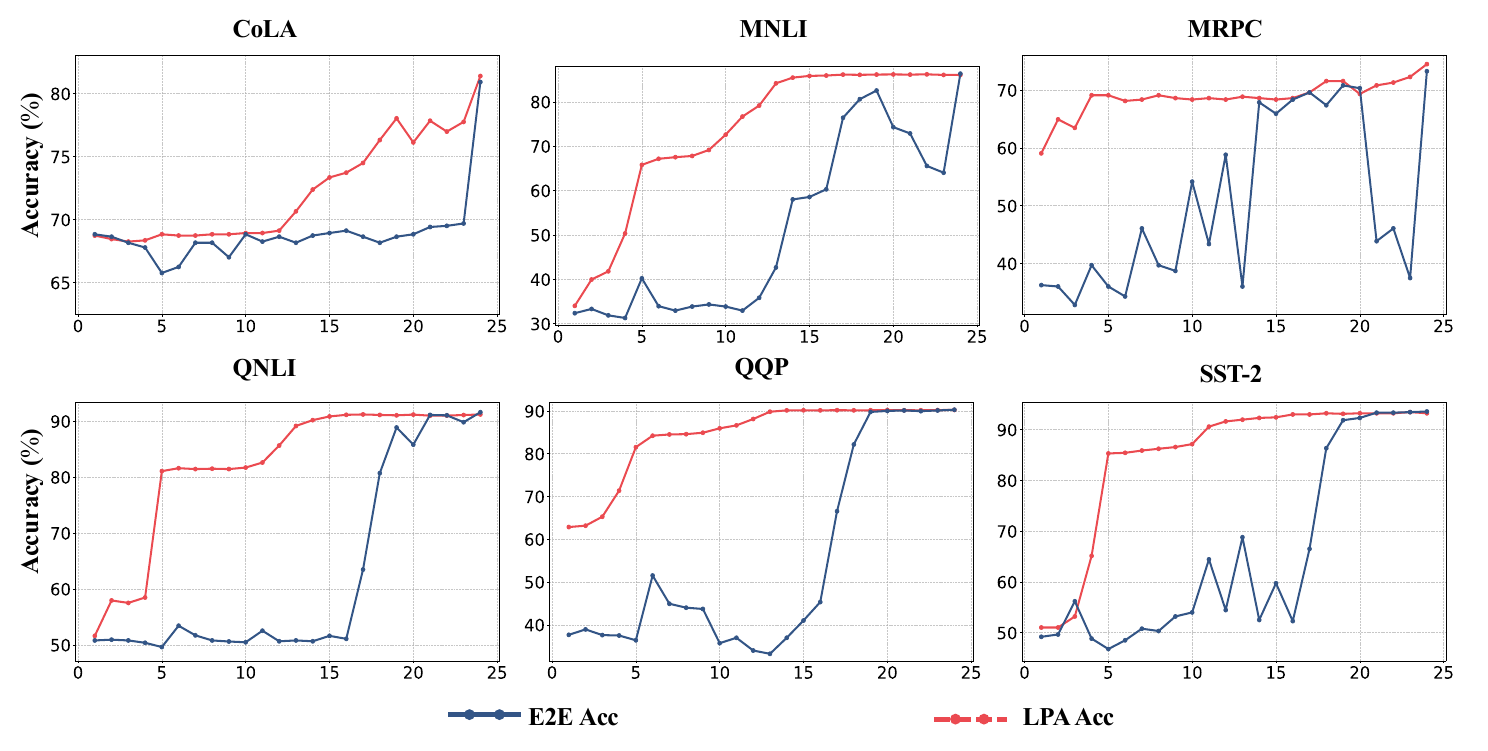} 
   \caption{Layer-wise performance comparison with Qwen. Convergence at early layers appears across all the LLM tasks.} 
   \label{fig:glue_cur} 
\end{figure}

\begin{figure}[ht]
   \centering
   \includegraphics[width=0.4\textwidth]{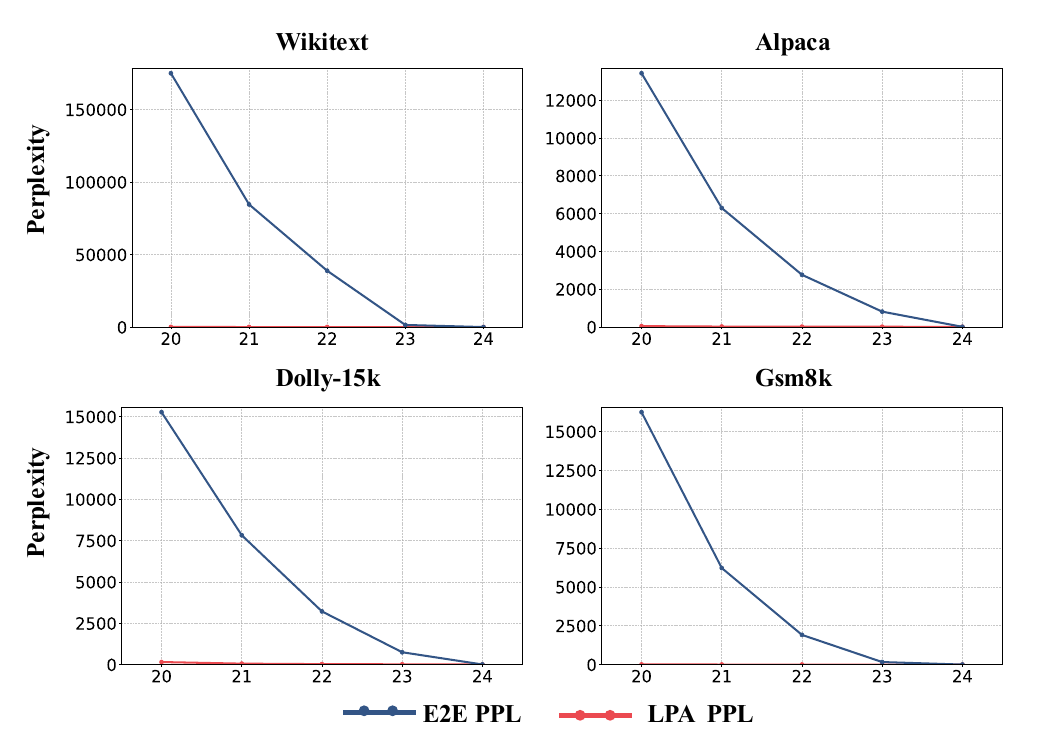} 
   \caption{Layer-wise perplexity comparison with Qwen. Convergence at early layers appears across all the LLM tasks.} 
   \label{fig:LLM_gen} 
\end{figure}

\begin{table*}[ht]
\centering
\caption{Layer-wise accuracy and loss of ViTs under E2E and LPA full training. For each layer, the higher accuracy is \textbf{bolded}.}
\label{tab:vit_imagenet_from_scratch}
\resizebox{0.85\textwidth}{!}{
\begin{tabular}{l|l|c|cccccccccccc}
\noalign{\hrule height 0.8pt}
\multicolumn{2}{c|}{\multirow{2}{*}{\textbf{Model / Method}}} & \multirow{2}{*}{\textbf{Metric}} &
\multicolumn{12}{c}{\textbf{Layers}} \\
\cline{4-15}
\multicolumn{2}{c|}{} & & 
\textbf{L2} & \textbf{L4} & \textbf{L6} & \textbf{L8} & \textbf{L10} & \textbf{L12} &
\textbf{L14} & \textbf{L16} & \textbf{L18} & \textbf{L20} & \textbf{L22} & \textbf{L24} \\
\hline
\multirow{4}{*}{\textbf{ViT-12}} 
& \multirow{2}{*}{\textbf{E2E}} & \textbf{Acc} $\uparrow$ & 0.12 & 0.24 & 0.35 & 5.18 & 33.54 & \textbf{78.42} & -- & -- & -- & -- & -- & -- \\
& & \textbf{Loss} $\downarrow$ & 6.87 & 6.83 & 6.70 & 6.27 & 5.78 & 1.06 & -- & -- & -- & -- & -- & -- \\
\cline{2-15}
& \multirow{2}{*}{\textbf{LPA}} & \textbf{Acc} $\uparrow$ & \textbf{15.95} & \textbf{51.34} & \textbf{66.79} & \textbf{73.32} & \textbf{77.44} & 78.29 & -- & -- & -- & -- & -- & -- \\
& & \textbf{Loss} $\downarrow$ & 4.67 & 2.34 & 1.53 & 1.24 & 1.05 & 1.14 & -- & -- & -- & -- & -- & -- \\
\hline
\multirow{4}{*}{\textbf{ViT-24}} 
& \multirow{2}{*}{\textbf{E2E}} & \textbf{Acc} $\uparrow$ & 0.10 & 0.11 & 0.14 & 0.13 & 0.15 & 0.12 & 0.11 & 0.19 & 2.43 & 12.29 & 49.38 & 76.70 \\
& & \textbf{Loss} $\downarrow$ & 7.01 & 6.918 & 6.89 & 6.87 & 6.875 & 6.876 & 6.874 & 6.84 & 6.57 & 6.30 & 5.11 & 1.30 \\
\cline{2-15}
& \multirow{2}{*}{\textbf{LPA}} & \textbf{Acc} $\uparrow$ & \textbf{42.42} & \textbf{60.37} & \textbf{70.52} & \textbf{74.79} & \textbf{76.42} & \textbf{79.08} & \textbf{80.18} & \textbf{80.51} & \textbf{80.67} & \textbf{80.60} & \textbf{79.90} & \textbf{79.76} \\
& & \textbf{Loss} $\downarrow$ & 2.88 & 1.89 & 1.37 & 1.17 & 1.11 & 1.01 & 1.02 & 1.060 & 1.07 & 1.06 & 1.16 & 1.21 \\
\noalign{\hrule height 0.8pt}
\end{tabular}}
\end{table*}

\begin{table}[ht]
\centering
\caption{Layer-wise accuracy and loss of ViT-Base and ViT-Small under E2E and LPA fine-tuning. For each layer, the higher accuracy between E2E and LPA is \textbf{bolded}.}
\label{tab:vit_layerwise_finetune30}
\resizebox{0.48\textwidth}{!}{
\begin{tabular}{l|l|c|cccccc} 
\noalign{\hrule height 0.8pt}
\multicolumn{2}{c|}{\multirow{2}{*}{\textbf{Model / Method}}} & \multirow{2}{*}{\textbf{Metric}} &
\multicolumn{6}{c}{\textbf{Layers}} \\ 
\cline{4-9}  
\multicolumn{2}{c|}{} & & 
\textbf{L2} & \textbf{L4} & \textbf{L6} & \textbf{L8} & \textbf{L10} & \textbf{L12} \\
\hline
\multirow{4}{*}{\textbf{ViT-Base}} 
& \multirow{2}{*}{\textbf{E2E}} & \textbf{Acc.} $\uparrow$ & 0.15 & 0.10 & 0.20 & 2.546 & 50.15 & \textbf{81.58} \\
& & \textbf{Loss} $\downarrow$ & 7.42 & 7.60 & 7.53 & 6.55 & 2.50 & 1.72 \\
\cline{2-9}
& \multirow{2}{*}{\textbf{LPA}} & \textbf{Acc.} $\uparrow$ & \textbf{5.69} & \textbf{12.012} & \textbf{23.87} & \textbf{47.13} & \textbf{69.95} & 81.56 \\
& & \textbf{Loss} $\downarrow$ & 5.81 & 5.11 & 4.13 & 2.698 & 1.462 & 0.859 \\
\hline
\multirow{4}{*}{\textbf{ViT-Small}} 
& \multirow{2}{*}{\textbf{E2e}} & \textbf{Acc.} $\uparrow$ & 0.10 & 0.10 & 0.10 & 0.63 & 25.69 & 77.08 \\
& & \textbf{Loss} $\downarrow$ & 7.84 & 7.69 & 7.57 & 6.97 & 4.90 & 1.15 \\
\cline{2-9}
& \multirow{2}{*}{\textbf{LPA}} & \textbf{Acc.} $\uparrow$ & \textbf{2.36} & \textbf{9.73} & \textbf{29.95} & \textbf{55.42} & \textbf{74.86} & \textbf{79.46} \\
& & \textbf{Loss} $\downarrow$ & 6.35 & 5.37 & 3.93 & 2.32 & 1.20 & 0.91 \\
\noalign{\hrule height 0.8pt}
\end{tabular}}
\end{table}

\subsection{Experiments for LLM tasks}

To evaluate the generality of progressive approximation, we test Qwen2-0.5B on both discriminative and generative tasks, examining whether layerwise refinement holds in modern transformer LLMs.
Note that our experiments are proof of concept rather than SOTA. In particular, LLM training is resource-intensive and beyond the scope of most labs, so we focus on feasibility rather than peak performance.

\textbf{Classification tasks.}  
We fine-tune Qwen2-0.5B on six GLUE datasets of CoLA, MNLI, MRPC, QNLI, QQP, and SST-2 \citep{Wang2018GLUEAM} with 3 epochs under both training paradigms. 
As shown in Figure~\ref{fig:glue_cur}, LPA yields a clear monotonic rise in layerwise accuracy, consistent with the theoretical progressive trajectory. In contrast, E2E shows fluctuations across layers, suggesting non-cumulative feature development.

\textbf{Generative and instruction-following tasks.}  
Given the higher complexity of generative tasks, we adopt a practical setting where only the final few transformer blocks are supervised in LPA. We evaluate on four diverse datasets: Wikitext (language modeling), Alpaca, Dolly-15k, and GSM8K (instruction-following and reasoning). All models are fine-tuned for 1 epoch.

Figure~\ref{fig:LLM_gen} reports the layer-wise perplexity (PPL) during inference using the shared readout head. Remarkably, LPA achieves substantially lower PPL even at shallow depths, and PPL continues to decrease smoothly as depth increases, demonstrating stable progressive refinement. By contrast, E2E exhibits significantly higher PPL in early and intermediate layers, with improvements concentrated only the output. This suggests that E2E fails to propagate meaningful supervision signals for generation to shallower components of the LLM.

Together, these results confirm that the progressive approximation principle can also be generalized to large-scale language architectures. 

\section{Conclusion}
In this work, we bridge approximation theory and modern deep learning by reinterpreting deep residual networks as a progressive approximation process. We prove that there exist layer-wise trajectories along which the approximation the target progressively with depth. It reveals that deep models can implement structured, step-by-step refinement rather than acting as opaque E2E mappers.
Building on this insight, we propose LPA, a training paradigm that explicitly enforces alignment between each layer and its residual target. Across a diverse range of architectures, including fully connected networks, ResNets, ViTs, and large language models, and tasks spanning scientific function approximation, image classification, and language understanding/generation, we consistently observe that LPA induces stable, cumulative refinement.

Critically, this property enables a powerful deployment paradigm: ``train once, use $N$ models". A single LPA-trained network provides meaningful predictions at every intermediate depth, allowing users to adapt model complexity to computational budgets without distillation, pruning, or retraining.
Our results suggest that progressive approximation is not merely a theoretical curiosity but a practical principle that unifies representation learning across modalities and scales. This offers both a new lens for understanding how deep networks compute and a versatile framework for efficient, depth-adaptive inference in real-world applications.

\section{Impact Statements}
This paper presents work whose goal is to advance the field of machine learning. There are many potential societal consequences of our work, none of which we feel must be specifically highlighted here.

\bibliography{example_paper}
\bibliographystyle{icml2026}

\newpage
\appendix
\onecolumn

\section{Unfold the RN}
\label{sec:UnfoldtheRN}

To analyze the representational structure of residual networks, we unfold the recursive definition into an explicit summation form.
A layer \(i\) of a RN can be expressed as:
\begin{equation}
\mathbf{x}_{i} = \mathbf{x}_{i-1} + \mathbf{W}'_{i} \sigma (\mathbf{W}_{i} \mathbf{x}_{i-1} + \mathbf{b}_{i}) + \mathbf{b}'_{i},
\label{eq:res_general_term}
\end{equation}
where \(\mathbf{x}_{i-1}\) and \(\mathbf{x}_{i}\) denote the input and output of the \(i\)-th residual block, respectively.
Defining the residual mapping as
\begin{equation}
G_i(\mathbf{x}_{i-1}) = \mathbf{W}'_{i} \sigma(\mathbf{W}_{i} \mathbf{x}_{i-1} + \mathbf{b}_{i}) + \mathbf{b}'_{i},
\end{equation}
which is a two-layer FNN, we rewrite Eq.~\eqref{eq:res_general_term} as
\begin{equation}
\mathbf{x}_{i} = \mathbf{x}_{i-1} + G_i(\mathbf{x}_{i-1}).
\end{equation}

Unfolding this recurrence yields:
\begin{align*}
\mathbf{x}_1 &= \mathbf{x}_0 + G_1(\mathbf{x}_0), \\
\mathbf{x}_2 &= \mathbf{x}_1 + G_2(\mathbf{x}_1) = \mathbf{x}_0 + G_1(\mathbf{x}_0) + G_2(\mathbf{x}_1), \\
&\;\;\vdots \\
\mathbf{x}_N &= \mathbf{x}_0 + \sum_{i=1}^{N} G_i(\mathbf{x}_{i-1}).
\end{align*}

\section{Approximation Trajectories}
\label{sec:Approximation_Trajectories}
To prepare for the formal analysis of progressive approximation, we first define the error trajectory induced by a trained residual network.

\begin{definition}[Error Trajectory]
    Let $f: K_0 \to K_N$ denote the target function to be approximated, and let $\mathbf{x}_0 \in K_0$ be an input.
    Consider an $N$-layer residual network that has been trained E2E, with forward recurrence
    \[
        \mathbf{x}_i = \mathbf{x}_{i-1} + G_i^*(\mathbf{x}_{i-1}), \quad i = 1, \dots, N,
    \]
    where each $G_i^*$ denotes the residual mapping after training.
    The approximation error after layer $i$ is defined as
    \begin{equation}
        \varepsilon_i := \big\| f(\mathbf{x}_0) - \mathbf{x}_i \big\|, \quad i = 0, 1, \dots, N.
    \end{equation}
    The sequence $\mathcal{T} = \{\varepsilon_i\}_{i=0}^N$ is called the error trajectory of the trained network on input $\mathbf{x}_0$.
\end{definition}

Note that the superscript $^*$ here signifies the mappings obtained from E2E optimization, not an ideal or optimal decomposition, but the actual functions realized by the trained model. Under the hypothesis of layer-wise progressive approximation, this trajectory would satisfy
\begin{equation}
    \varepsilon_0 \geq \varepsilon_1 \geq \cdots \geq \varepsilon_N,
\end{equation}
meaning that each residual block brings the representation strictly closer to the target.

However, standard end-to-end training imposes no such monotonicity constraint. Instead, all residual mappings $\{G_i^*\}_{i=1}^N$ are jointly optimized solely to minimize the final loss $\varepsilon_N = \|f(\mathbf{x}_0) - \mathbf{x}_N\|$. Consequently, intermediate errors $\varepsilon_i$ for $i < N$ are not directly supervised, and the resulting error trajectory $\mathcal{T}$ may exhibit non-monotonic behavior.

To illustrate, consider:
\begin{itemize}
    \item \textbf{Single-layer RN}: $\mathbf{x}_1 = \mathbf{x}_0 + G_1^*(\mathbf{x}_0)$, so $\varepsilon_1 = \|f(\mathbf{x}_0) - \mathbf{x}_0 - G_1^*(\mathbf{x}_0)\|$. Although $G_1^*$ contributes to reducing $\varepsilon_N$, it may inadvertently increase $\varepsilon_1$ if doing so facilitates later corrections.
    
    \item \textbf{Two-layer RN}: $\mathbf{x}_2 = \mathbf{x}_0 + G_1^*(\mathbf{x}_0) + G_2^*(\mathbf{x}_1)$, and
    \[
        \varepsilon_2 = \big\| f(\mathbf{x}_0) - \mathbf{x}_0 - G_1^*(\mathbf{x}_0) - G_2^*(\mathbf{x}_1) \big\|.
    \]
    It is entirely possible that $\varepsilon_2 < \varepsilon_1$ but $\varepsilon_1 > \varepsilon_0$, or even that $\varepsilon_2 > \varepsilon_1$, depending on how $G_1^*$ and $G_2^*$ co-adapt during training.
\end{itemize}

In general, for $i$ layers,
\[
\varepsilon_i = \Big\| f(\mathbf{x}_0) - \mathbf{x}_0 - \sum_{j=1}^{i} G_j^*(\mathbf{x}_{j-1}) \Big\|,
\]
and the lack of local supervision allows the trajectory $\{\varepsilon_i\}$ to:
\begin{itemize}
    \item temporarily increase due to over-correction or representation reconfiguration;
    \item stagnate when updates become redundant;
    \item decrease only after sufficient representational groundwork has been laid.
\end{itemize}

Thus, standard training does not constrain the approximation trajectory. The actual error trajectory is an emergent property of global optimization.

\section{Residual Networks with Dimension Changes}
\subsection{Residual Networks with Shared Linear Readout}
\label{sec:SharedReadout}
In practical residual networks, the hidden dimension $d$ often differs from the target output dimension $d_y$. To handle this mismatch, a linear readout layer $\mathbf{W}_{N+1} \in \mathbb{R}^{d_y \times d}$ is applied to the final hidden state:
\begin{equation}
\hat{f}_{N+1}(\mathbf{x}_0) = \mathbf{W}_{N+1} \mathbf{x}_N.
\end{equation}
Due to the additive structure of residual connections,
\[
\mathbf{x}_N = \mathbf{x}_0 + G_1(\mathbf{x}_0) + G_2(\mathbf{x}_1) + \cdots + G_N(\mathbf{x}_{N-1}),
\]
the network output can be equivalently expressed as
\begin{equation}
\hat{f}_{N+1}(\mathbf{x}_0) = \mathbf{W}_{N+1} \mathbf{x}_0 + \sum_{i=1}^N \underbrace{\mathbf{W}_{N+1} G_i(\mathbf{x}_{i-1})}_{=: H_i(\mathbf{x}_{i-1})}.
\end{equation}
Thus, each residual block contributes to the final prediction through the same linear map $\mathbf{W}_{N+1}$. This allows us to analyze the approximation process directly in the output space $\mathbb{R}^{d_y}$. Note that each $G_i: \mathbb{R}^d \to \mathbb{R}^d$ is a FNN, and therefore the composition
\[
H_i(\mathbf{x}) := \mathbf{W}_{N+1} G_i(\mathbf{x})
\]
is also a FNN mapping $\mathbb{R}^d \to \mathbb{R}^{d_y}$. By the UAT, such networks can approximate any continuous function on a compact set arbitrarily well, provided sufficient width. Hence, the layer-wise approximation analysis from Section~\ref{subsec:prog_traj} extends naturally to this setting.

We now extend the existence result of Theorem~\ref{th:prog_traj} to the setting where predictions are formed via a shared linear readout $\mathbf{W}_{N+1} \in \mathbb{R}^{d_y \times d}$, i.e., $\hat{f}_i(\mathbf{x}_0) = \mathbf{W}_{N+1} \mathbf{x}_i$, even when $d \neq d_y$. We show that a non-increasing error trajectory in output space still exists.

Let $f: K_0 \to \mathbb{R}^{d_y}$ be continuous on a compact set $K_0 \subset \mathbb{R}^d$, and define the layer-wise prediction as
\[
\hat{f}_i(\mathbf{x}_0) := \mathbf{W}_{N+1} \mathbf{x}_i,
\quad \text{where} \quad
\mathbf{x}_i = \mathbf{x}_{i-1} + G_i(\mathbf{x}_{i-1}), \quad \mathbf{x}_0 \in K_0.
\]
The approximation error at depth $i$ is
\[
\varepsilon_i := \sup_{\mathbf{x}_0 \in K_0} \big\| f(\mathbf{x}_0) - \hat{f}_i(\mathbf{x}_0) \big\|.
\]

We construct $\{G_i\}_{i=1}^N$ inductively such that $\varepsilon_1 \geq \varepsilon_2 \geq \cdots \geq \varepsilon_N \geq 0$.

\textbf{Base case ($i=1$).}  
Define the initial prediction $\hat{f}_0(\mathbf{x}_0) = \mathbf{W}_{N+1} \mathbf{x}_0$ and the first residual target
\[
f_1(\mathbf{x}_0) := f(\mathbf{x}_0) - \hat{f}_0(\mathbf{x}_0).
\]
Since $f_1$ is continuous on the compact set $K_0$, the UAT guarantees the existence of a FNN $G_1^*: \mathbb{R}^d \to \mathbb{R}^d$ such that the induced correction
\[
H_1^*(\mathbf{x}_0) := \mathbf{W}_{N+1} G_1^*(\mathbf{x}_0)
\]
satisfies
\[
\sup_{\mathbf{x}_0 \in K_0} \big\| f_1(\mathbf{x}_0) - H_1^*(\mathbf{x}_0) \big\| < \varepsilon_1,
\]
for some $\varepsilon_1 \geq 0$. Setting $\mathbf{x}_1 = \mathbf{x}_0 + G_1^*(\mathbf{x}_0)$, we obtain
\[
\hat{f}_1(\mathbf{x}_0) = \mathbf{W}_{N+1} \mathbf{x}_1 = \hat{f}_0(\mathbf{x}_0) + H_1^*(\mathbf{x}_0),
\]
and thus $\varepsilon_1 = \sup_{\mathbf{x}_0} \|f(\mathbf{x}_0) - \hat{f}_1(\mathbf{x}_0)\|$ is well-defined.

\textbf{Inductive step.}  
Assume that after $i-1$ layers, we have constructed $\{G_j^*\}_{j=1}^{i-1}$ such that
\[
\hat{f}_{i-1}(\mathbf{x}_0) = \mathbf{W}_{N+1} \mathbf{x}_{i-1}
= \mathbf{W}_{N+1} \mathbf{x}_0 + \sum_{j=1}^{i-1} H_j^*(\mathbf{x}_{j-1}),
\]
with error $\varepsilon_{i-1} = \sup_{\mathbf{x}_0} \|f(\mathbf{x}_0) - \hat{f}_{i-1}(\mathbf{x}_0)\|$. Define the remaining residual
\[
f_i(\mathbf{x}_0) := f(\mathbf{x}_0) - \hat{f}_{i-1}(\mathbf{x}_0).
\]
Note that $f_i$ is a continuous function of $\mathbf{x}_0$, and since $\mathbf{x}_{i-1}$ is a continuous function of $\mathbf{x}_0$ (as a composition of continuous maps), we may view $f_i$ as a continuous function of $\mathbf{x}_{i-1}$ over the compact image set $K_{i-1} := \{\mathbf{x}_{i-1}(\mathbf{x}_0) : \mathbf{x}_0 \in K_0\} \subset \mathbb{R}^d$.

By the UAT, there exists a network $G_i^*: \mathbb{R}^d \to \mathbb{R}^d$ such that the correction
\[
H_i^*(\mathbf{x}_{i-1}) := \mathbf{W}_{N+1} G_i^*(\mathbf{x}_{i-1})
\]
approximates $f_i$ arbitrarily well:
\[
\sup_{\mathbf{x}_0 \in K_0} \big\| f_i(\mathbf{x}_0) - H_i^*(\mathbf{x}_{i-1}) \big\| < \varepsilon_i,
\]
for some $\varepsilon_i \geq 0$. Crucially, we can always choose $G_i^*$ such that $\varepsilon_i \leq \varepsilon_{i-1}$—for instance, by selecting $G_i^* \equiv 0$ yields $\varepsilon_i = \varepsilon_{i-1}$, while any better approximator gives $\varepsilon_i < \varepsilon_{i-1}$.

Setting $\mathbf{x}_i = \mathbf{x}_{i-1} + G_i^*(\mathbf{x}_{i-1})$, the new prediction becomes
\[
\hat{f}_i(\mathbf{x}_0) = \hat{f}_{i-1}(\mathbf{x}_0) + H_i^*(\mathbf{x}_{i-1}),
\]
and the error satisfies $\varepsilon_i \leq \varepsilon_{i-1}$.

By induction, there exists a sequence $\{G_i^*\}_{i=1}^N$ such that
\[
\varepsilon_1 \geq \varepsilon_2 \geq \cdots \geq \varepsilon_N \geq 0.
\]

This establishes that progressive approximation holds in the output space even under a shared linear readout. The key requirement is that each residual $f_i$ remains continuous in the current representation $\mathbf{x}_{i-1}$, which is guaranteed by the continuity of the network up to layer $i-1$ and the compactness of $K_0$. The shared matrix $\mathbf{W}_{N+1}$ merely defines a fixed linear projection into $\mathbb{R}^{d_y}$ and does not impede the layer-wise refinement enabled by the residual structure and UAT.


\subsection{Layer-wise Representation under Intermediate Dimension Changes}
\label{sec:dim_change_proof}

Consider a residual network with possible dimension changes across layers. Let the forward recurrence be
\[
\mathbf{x}_i = \mathbf{W}_{i-1}^{\mathrm{proj}} \mathbf{x}_{i-1} + G_i(\mathbf{x}_{i-1}), \quad i = 1, \dots, N,
\]
where $\mathbf{x}_i \in \mathbb{R}^{d_i}$, $\mathbf{W}_{i-1}^{\mathrm{proj}} \in \mathbb{R}^{d_i \times d_{i-1}}$ is a linear projection (identity if $d_i = d_{i-1}$), and $G_i: \mathbb{R}^{d_{i-1}} \to \mathbb{R}^{d_i}$ is the residual function.

Let the target be $f(\mathbf{x}_0) \in \mathbb{R}^{d_y}$. To define layer-wise predictions in a common output space despite heterogeneous intermediate dimensions, we use the full linear path from layer $i$ to the output. Specifically, define the post-$i$ linear operator
\[
\widetilde{\mathbf{W}}_i := \mathbf{W}_{N+1} \mathbf{W}_{N-1}^{\mathrm{proj}} \mathbf{W}_{N-2}^{\mathrm{proj}} \cdots \mathbf{W}_{i}^{\mathrm{proj}} \in \mathbb{R}^{d_y \times d_i},
\]
where $\mathbf{W}_{N+1} \in \mathbb{R}^{d_y \times d_N}$ is a learnable readout matrix at  the last layer. The prediction at depth $i$ is then defined as
\[
\hat{f}_i(\mathbf{x}_0) := \widetilde{\mathbf{W}}_i \mathbf{x}_i.
\]
Note that $\mathbf{W}_{N+1}$ and all $\{\mathbf{W}_{j}^{\mathrm{proj}}\}$ are jointly optimized during training.

Unrolling the recurrence up to layer $i$, we obtain
\[
\mathbf{x}_i = \mathbf{W}_{i-1}^{\mathrm{proj}} \cdots \mathbf{W}_{0}^{\mathrm{proj}} \mathbf{x}_0 
+ \sum_{j=1}^i \left( \mathbf{W}_{i-1}^{\mathrm{proj}} \cdots \mathbf{W}_{j}^{\mathrm{proj}} \right) G_j(\mathbf{x}_{j-1}),
\]
with the convention that an empty product equals the identity. Applying $\widetilde{\mathbf{W}}_i$ yields
\begin{equation}
\begin{aligned}
\hat{f}_i(\mathbf{x}_0) 
&= \widetilde{\mathbf{W}}_i \mathbf{x}_i \\
&= \underbrace{\mathbf{W}_{N+1} \mathbf{W}_{N-1}^{\mathrm{proj}} \cdots \mathbf{W}_{0}^{\mathrm{proj}}}_{=: \widetilde{\mathbf{W}}_0} \mathbf{x}_0 
+ \sum_{j=1}^i \underbrace{\mathbf{W}_{N+1} \mathbf{W}_{N-1}^{\mathrm{proj}} \cdots \mathbf{W}_{j}^{\mathrm{proj}} G_j}_{=: \mathbf{H}_j}(\mathbf{x}_{j-1}) \\
&= \widetilde{\mathbf{W}}_0 \mathbf{x}_0 + \sum_{j=1}^i H_j(\mathbf{x}_{j-1}),
\end{aligned}
\label{eq:layer_rep_general}
\end{equation}
where each $H_j(\cdot) := \mathbf{W}_{N+1} \mathbf{W}_{N-1}^{\mathrm{proj}} \cdots \mathbf{W}_{j}^{\mathrm{proj}} G_j: \mathbb{R}^{d_{j-1}} \to \mathbb{R}^{d_y}$ represents the contribution of block $j$ in the target output space.

Since all $\mathbf{W}_j$ are linear maps , the entire expression is a well-defined sum in $\mathbb{R}^{d_y}$. Moreover, because matrix multiplication is associative, each $\widetilde{\mathbf{W}}_i$ can be fused into a single matrix during inference. This ensures that the sequence $\{\hat{f}_i(\mathbf{x}_0)\}_{i=0}^N$ lies in a common space $\mathbb{R}^{d_y}$, enabling a consistent error trajectory $\varepsilon_i = \| f(\mathbf{x}_0) - \hat{f}_i(\mathbf{x}_0) \|$ even when intermediate dimensions vary.

\section{Data for Complex Surface Fitting}
\label{sec:Curve_Fitting_data}
We generate 6 synthetic datasets using functions as follows. There are 20,000 input-output pairs for each target function (10,000 pairs for training and 10,000 pairs for testing). 
\begin{itemize}
\item $ f(x, y) = \big( |x| ^{0.7} + |y| ^{1.3}\big) \sin(4x) $ (Asymmetric power-modulated waves): 
A distorted oscillatory surface with non-integer, directionally unbalanced amplitude envelopes; tests the model’s ability to fit asymmetric scaling combined with linear-frequency sine modulation.

\item $ f(x, y) = e^{-5(x^2 + y^2)} \sin(10x) $ (Localized high-frequency stripes): A sharply localized Gaussian window containing multiple rapid oscillations along the \(x\)-axis; evaluates joint modeling of spatial confinement and fine-scale directional features within a bounded domain.

\item $ f(x, y) = e^{x^2 - y^2} $ (Exponential saddle): A gently varying saddle-shaped surface with exponential contrast between orthogonal axes; assesses capture of hyperbolic structure under mild.

\item $ f(x, y) = \log(x^2 + y^2 + 10^{-5}) \cos(5x) $ (Logarithmic central well): An oscillating surface featuring a deep, narrow depression at the origin due to a regularized logarithmic singularity; probes robustness to near-singular radial potentials coupled with transverse cosine waves.

\item $ f(x, y) = \sin(3x) \cos(3y) $ (Standing wave grid): A regular checkerboard-like interference pattern with uniform spatial frequency; serves as a baseline for testing periodicity, orthogonality, and separable oscillation representation.

\item $ f(x, y) = \dfrac{\sin(4x^2 + y^2)}{\sqrt{x^2 + y^2 + 0.001}} $ (Anisotropic radial oscillations): A directionally asymmetric ripple pattern whose frequency increases quadratically away from the origin, modulated by a near-singular radial amplitude; challenges modeling of anisotropic behavior with weak singularity regularization.

\item $ f(x, y) = \sin(5x^2) + \cos(3y^2) $ (Quadratically warped oscillations): 
Non-uniform spatial frequencies that increase with distance from the origin in each axis; tests the ability to learn separable, nonlinearly warped oscillatory structures.

\item $ f(x, y) = \dfrac{\sin\!\big(10(x^2 + y^2)\big)}{x^2 + y^2 + 0.1} $ (Damped concentric ripples):
A smooth, circularly symmetric oscillatory surface with amplitude decaying radially; evaluates handling of rotational symmetry and damped high-frequency content in a compact domain.
\end{itemize}

\section{Unifing the Representation of Resiudal CNN and Transformer}
\label{apd:sec:unify_RNs}
As we establish early in the main text, if a residual network layer can be written as Eq.~(\ref{eq:res_general_term})
\begin{equation}
\mathbf{x}_{i} = \mathbf{x}_{i-1} + \mathbf{W}'_{i} \sigma (\mathbf{W}_{i} \mathbf{x}_{i-1} + \mathbf{b}_{i}) + \mathbf{b}'_{i},
\end{equation}
then the layer-wise UAT holds, and LPA is applicable.
In this section, we detail how CNNs and Transformers can be expressed in this form.

\subsection{Matrix and Vecotr Representations}
Before diving into the details, we emphasize that in the equation, $\mathbf{W}'_{i}$ and $\mathbf{W}_{i}$ are weight matricies, while the $\mathbf{x}_{i}$, $\mathbf{x}_{i-1}$, $\mathbf{b}_{i}$ and $\mathbf{b}'_{i}$ are vectors.
In the derivations below, all products strictly follow the standard definitions of linear algebra and matrix analysis, from a mathematical standpoint, rather than the simplified conventions often used in machine learning to describe inputs, outputs, and weights.

\subsection{Conversion for CNNs}
It is evident that convolution operations can be equivalently expressed in matrix-vector form. This transformation allows us to interpret the convolution process as applying a structured, sparse matrix (derived from the convolutional kernel) to the input vector. In this formulation, each row of the matrix corresponds to a local receptive field, and the sparsity pattern encodes the weight sharing and local connectivity inherent in convolutional layers. This perspective not only unifies convolution with linear transformations but also highlights its role in hierarchically extracting spatial features from the input data. For a detailed theoretical treatment of this equivalence and its implications for representational capacity refer to \cite{zhou2020universality}. Based on the matrix-vector form of convolution, it is easy to know that each layer of a multi-layer residual based CNN can be expressed as Eq.~(\ref{eq:res_general_term}).

\subsection{Conversion for Transformers}
\label{sec:CFT}
Representing Transformer operations in matrix–vector form is more involved, since they are typically expressed in matrix form (e.g., $\mathbf{Y}=\mathbf{W}\mathbf{X}$).
Let us define a basic conversion $\mathbf{x}=vec(\mathbf{X})$ and $\mathbf{y}=vec(\mathbf{Y})$, where $vec(\cdot)$ denotes vectorization.
\begin{equation}
\begin{aligned}
\mathbf{X}=&
\begin{vmatrix}
x_{11} & x_{12} & \cdots & x_{1n}\\
x_{21} & x_{22} & \cdots & x_{2n}\\
\vdots & \vdots & \vdots & \vdots\\
x_{m1} & x_{m2} & \cdots & x_{mn}\\
\end{vmatrix}_{m\times n}\\
\Rightarrow
\mathbf{x}=&
\left[
\begin{array}{ccccccccccccc}
x_{11} & x_{21} & \cdots & x_{m1} & x_{12} & x_{22} & \cdots & x_{m2} & \cdots & x_{1n} & x_{2n} & \cdots & x_{mn}
\end{array}
\right]^\top_{1 \times mn}.
\end{aligned}
\end{equation}
However, as we will soon see, this alone does not suffice to translate products like $\mathbf{Y}=\mathbf{W}\mathbf{X}$, $\mathbf{Y}=\mathbf{X}\mathbf{W}$, or $\mathbf{Y}=\mathbf{W}\mathbf{X}\mathbf{W}$ into equivalent matrix–vector expressions. 
We first need additional, more fundamental conversion tools.

\subsubsection{Fundemental Conversion Tools}

\begin{theorem}[\textbf{Right Multiplication Conversion to Matrix-Vector Form}]
\label{th:right_conv}
With $\mathbf{X}\in\mathbb{R}^{m\times n}, \mathbf{Y}\in\mathbb{R}^{m\times p}, \mathbf{W}\in\mathbb{R}^{n\times p}$, the right matrix multiplicaion $\mathbf{Y} = \mathbf{X} \mathbf{W}$ can be converted into the form of $\mathbf{y} = \mathbf{W}'\mathbf{x}$
where $\mathbf{W}'= \mathbf{W}^\top \otimes \mathbf{I}_m$ with  $\mathbf{I}_m\in \mathbb{R}^{m \times m}$ representing a unit matrix.
\end{theorem}

\begin{proof}
Starting with
\begin{equation}
    \mathbf{Y}_{m\times p} = \mathbf{X}_{m\times n} \mathbf{W}_{n\times p},
\end{equation}
we can write 
\begin{align}
    vec(\mathbf{Y}_{m\times p})&=vec(\mathbf{X}_{m\times n} \mathbf{W}_{n\times p})\nonumber\\
    &=vec(\mathbf{I}_{m\times m}\mathbf{X}_{m\times n} \mathbf{W}_{n\times p})\nonumber\\
    &=(\mathbf{W}^\top \otimes \mathbf{I}_m) vec(\mathbf{X})\label{eq:vect_prop}\\
    &=(\mathbf{W}^\top \otimes \mathbf{I}_m) \mathbf{x}\nonumber\\
    &=\mathbf{W}'\mathbf{x}
\end{align}
where $\otimes$ is Kronecker product, and
$\mathbf{W}' = \mathbf{W}^\top \otimes \mathbf{I}_m$.
In Eq.~(\ref{eq:vect_prop}), Vectorization Property, which states $vec(\mathbf{ABC})=(\mathbf{C}^\top\otimes\mathbf{A})vec(\mathbf{B})$ is used.
\end{proof}

In a more general case, we often have a bias term as $\mathbf{Y} = \mathbf{X} \mathbf{W} + \mathbf{B}_{1\times m}$. 
It is easy to see that its vertor form is:
$$
\mathbf{y} = \mathbf{W}' \mathbf{x} + \mathbf{b},
$$
where $\mathbf{b}=\mathbf{B}\otimes \mathbf{1}_p$ and $\mathbf{1}_m \in \mathbb{R}^{m \times 1}$ is a column vector of all ones.

\begin{theorem}[\textbf{Left Multiplication Conversion to Matrix-Vector Form}]
\label{right_conv}
With $\mathbf{X}\in\mathbb{R}^{m\times n}, \mathbf{Y}\in\mathbb{R}^{p\times n}, \mathbf{W}\in\mathbb{R}^{p\times m}$, the left matrix multiplicaion $\mathbf{Y} = \mathbf{W} \mathbf{X} $ can be converted into the form of $\mathbf{y} = \mathbf{W}'\mathbf{x}$
where $\mathbf{W}'=\mathbf{I}_n \otimes \mathbf{W} $ with  $\mathbf{I}_n\in \mathbb{R}^{n \times n}$ representing a unit matrix.
\end{theorem}

\begin{proof}
Starting with
\begin{equation}
    \mathbf{Y}_{p\times n} = \mathbf{W}_{p\times m}\mathbf{X}_{m\times n} ,
\end{equation}
we can write 
\begin{align}
    vec(\mathbf{Y}_{p\times n})&=vec(\mathbf{W}_{p\times m}\mathbf{X}_{m\times n} )\nonumber\\
    &=vec(\mathbf{W}_{p\times m}\mathbf{X}_{m\times n} \mathbf{I}_{n})\nonumber\\
    &=(\mathbf{I}_{n} \otimes \mathbf{W}_{p\times m}) vec(\mathbf{X})\label{eq:vect_prop_left}\\
    &=(\mathbf{I}_{n} \otimes \mathbf{W}_{p\times m}) \mathbf{x}\nonumber\\
    &=\mathbf{W}'\mathbf{x}
\end{align}
where 
$\mathbf{W}' = \mathbf{I}_{n} \otimes \mathbf{W}$.
In Eq.~(\ref{eq:vect_prop_left}), Vectorization Property is used.
\end{proof}

\begin{theorem}[\textbf{Bi-Directional Multiplication Conversion to Matrix-Vector Form}]
\label{bi_conv}
With $\mathbf{X}\in\mathbb{R}^{n\times p}, \mathbf{Y}\in\mathbb{R}^{m\times q}, \mathbf{W}_L\in\mathbb{R}^{m\times n}, \mathbf{W}_L\in\mathbb{R}^{p\times q}$, the matrix multiplicaion $\mathbf{Y} = \mathbf{W}_L \mathbf{X} \mathbf{W}_R$ can be converted into the form of $\mathbf{y} = \mathbf{W}'\mathbf{x}$
where $\mathbf{W}'=\mathbf{W}_R^\top \otimes \mathbf{W}_L$.
\end{theorem}

\begin{proof}
Starting with
\begin{equation}
    \mathbf{Y} = \mathbf{W}_L \mathbf{X} \mathbf{W}_R,
\end{equation}
where $\mathbf{W}_L \in \mathbb{R}^{m \times n}$, $\mathbf{X} \in \mathbb{R}^{n \times p}$, $\mathbf{W}_R \in \mathbb{R}^{p \times q}$, and $\mathbf{Y} \in \mathbb{R}^{m \times q}$.  
We proceed in two steps:
\begin{itemize}
    \item Let $\mathbf{Z} = \mathbf{W}_L \mathbf{X}$, then $vec(\mathbf{Z})=\mathbf{z} = (\mathbf{I}_p \otimes \mathbf{W}_L) \mathbf{x}$.
    \item Then, $\mathbf{Y} = \mathbf{Z} \mathbf{W}_R$, so $vec(\mathbf{Y})=\mathbf{y} = (\mathbf{W}_R^\top \otimes \mathbf{I}_m) \mathbf{z}$.
\end{itemize}

Substituting the expression for $\mathbf{z}$, we get:
$$
\mathbf{y} = (\mathbf{W}_R^T \otimes \mathbf{I}_m)(\mathbf{I}_p \otimes \mathbf{W}_L) \mathbf{x}.
$$

Using the mixed-product property of the Kronecker product:
$$
(\mathbf{A} \otimes \mathbf{B})(\mathbf{C} \otimes \mathbf{D}) = (\mathbf{A}\mathbf{C}) \otimes (\mathbf{B}\mathbf{D}),
$$
we simplify:
$$
(\mathbf{W}_R^\top \otimes \mathbf{I}_m)(\mathbf{I}_p \otimes \mathbf{W}_L) = \mathbf{W}_R^\top \otimes \mathbf{W}_L.
$$

Thus, the final expression becomes:
$$
\mathbf{y} = (\mathbf{W}_R^\top \otimes \mathbf{W}_L) \mathbf{x}.
$$
\end{proof}

\subsubsection{Formulation for Transformers}

\begin{figure}[htbp!]
\centering
\includegraphics[width=0.48\textwidth]{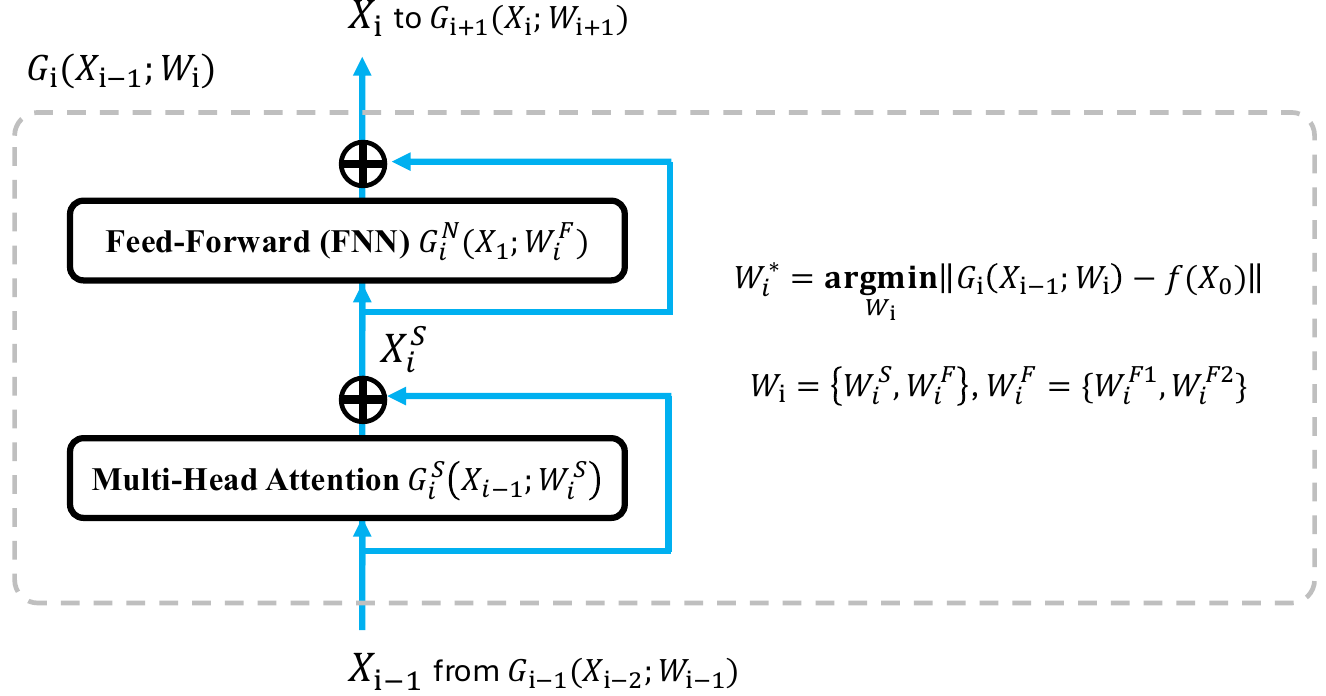}
\caption{The $i^{th}$ layer of a Transformer illustrated using the notations defined in this paper.}
\label{fig:trans_layer}
\end{figure}

Figure~\ref{fig:trans_layer} depicts the $i^{th}$ Transformer layer, consisting of a Multi-Head Attention (MHA) block followed by a Feed-Forward (FFN) block. We now cast this process into matrix–vector form.

\textbf{Multi-Head Attention Formulation} The input $\mathbf{X}_{i-1}$ is a matrix with each column representing an input token embedding, which will be transformed in to query, key, value matrices, respectively as at the $j^{th}$ self-attention head as
\begin{gather*}
   \mathbf{Q}_{i,j}=\mathbf{W}^Q_{i,j}\mathbf{X}_{i-1},\\
   \mathbf{K}_{i,j}=\mathbf{W}^K_{i,j}\mathbf{X}_{i-1},\\
   \mathbf{V}_{i,j}=\mathbf{W}^V_{i,j}\mathbf{X}_{i-1}, 
\end{gather*}
where $\mathbf{X}_{i-1} \in \mathbb{R}^{d \times I} , \mathbf{W}^Q_{i,j}, \mathbf{W}^K_{i,j}, \mathbf{W}^V_{i,j} \in \mathbb{R}^{dh \times d}$, and $h$ is the total number of heads.
It is then processed as
\begin{equation}
G_{i,j}^S(\mathbf{X}_{i-1})=\mathbf{V}_{i,j}softmax\left(\frac{\mathbf{Q}_{i,j}^T\mathbf{K}_{i,j}}{\sqrt{d}}\right)
\end{equation}
Letting $\mathbf{H}_{i,j}=softmax\left(\frac{\mathbf{Q}_{i,j}^T\mathbf{K}_{i,j}}{\sqrt{d}}\right)$, we have \begin{equation}
    G_{i,j}^S(\mathbf{X}_{i-1})=\mathbf{W}_{i,j}^V\mathbf{X}_i\mathbf{H}_{i,j}.
    \label{eq:attention}
\end{equation} 

Using \textbf{Theorem~\ref{bi_conv}}, We can the $j$-th self-attention in $i$-th layer can be written as:
\begin{equation}
vec(G_{i,j}^S(\mathbf{X}_{i-1}))=g_{i,j}^S(\mathbf{x}_{i-1})=\left(\mathbf{H}_{i,j}^\top \otimes \mathbf{W}_{i,j}^V \right) \mathbf{x}_{i-1}
\end{equation} 
The whole MHA computation with all $h$ heads merged is then written
\begin{equation}
\begin{aligned}
    G_{i}^S(\mathbf{X}_{i-1})&=\mathbf{W}_i^O\text{Concat}[G_{i,1}^S(\mathbf{X}_{i-1})^\top, \cdots G_{i,h}^S(\mathbf{X}_{i-1})^\top]^\top\\
    &+\mathbf{X}_{i-1}.
\end{aligned}
\end{equation}
where $\mathbf{W}_{i}^O \in \mathbb{R}^{d \times d}$ is the weight matrix for head merger.
This a right matrix multiplication, by applying \textbf{Theorem~\ref{th:right_conv}}, it can be exprssed in vectorized form as
\begin{equation}
g_{i}^S(\mathbf{x}_{i-1}) = (\mathbf{I}_I \otimes \mathbf{W}_{i,O})
\begin{bmatrix}
\left(\mathbf{H}_{i,1}^\top \otimes \mathbf{W}_{i,1}^V\right) \mathbf{x}_{i-1} \\
\left(\mathbf{H}_{i,2}^\top \otimes \mathbf{W}_{i,2}^V\right) \mathbf{x}_{i-1} \\
\vdots \\
\left(\mathbf{H}_{i,h}^\top \otimes \mathbf{W}_{i,h}^V\right) \mathbf{x}_{i-1}
\end{bmatrix}
+ \mathbf{x}_{i-1}.
\label{eq:vect_MHA}
\end{equation}
Note that we have replaced $G^S()$ with $g^S()$ after vectorization, as this adheres more strictly to the mathematical convention. 
Specifically, $G^S()$ operates on the matrix-form input $\mathbf{X}_{i-1}$, while $g^S()$ operates on the vector-form input $\mathbf{x}_{i-1}$.

Eq.~(\ref{eq:vect_MHA}) can be simplified as:
\begin{equation}
\begin{aligned}
&vec(G_{i}^S(\mathbf{X}_{i-1}))=g_{i}^S(\mathbf{x}_{i-1})= 
& (\mathbf{I}_I \otimes \mathbf{W}_{i,O})
\begin{bmatrix}
\left(\mathbf{H}_{i,1}^\top \otimes \mathbf{W}_{i,1}^V\right)  \\
\left(\mathbf{H}_{i,2}^\top \otimes \mathbf{W}_{i,2}^V\right)  \\
\vdots \\
\left(\mathbf{H}_{i,h}^\top \otimes \mathbf{W}_{i,h}^V\right) 
\end{bmatrix}\mathbf{x}_{i-1}
+ \mathbf{x}_{i-1}.
\end{aligned}
\end{equation}

Let
$$
\mathbf{W}_{i}^S = (\mathbf{I}_I \otimes \mathbf{W}_{i,O})
\begin{bmatrix}
\left(\mathbf{H}_{i,1}^\top \otimes \mathbf{W}_{i,1}^V\right)  \\
\left(\mathbf{H}_{i,2}^\top \otimes \mathbf{W}_{i,2}^V\right)  \\
\vdots \\
\left(\mathbf{H}_{i,h}^\top \otimes \mathbf{W}_{i,h}^V\right) 
\end{bmatrix},
$$
the whole process of MHA in the $i$-th layer can be written into vector form as
\begin{equation}
 g_{i}^S(\mathbf{x}_{i-1})=\mathbf{W}_{i}^S\mathbf{x}_{i-1}
+ \mathbf{x}_{i-1} .   
\end{equation}

\textbf{Feed-Forward Formulation} 
\label{sec:ffn}
The FNN component is essentially a FCN which takes input $\mathbf{X}_i^S=G^S(\mathbf{X}_{i-1})$ (or in vector form $\mathbf{x}_i^S=g^S(\mathbf{x}_{i-1})$) from the MHA component, it can be written as
\begin{equation} \mathbf{X}_i=G_i^F(\mathbf{X}_i^S;\mathbf{W}_i^F)=\mathbf{W}_i^{F2}\sigma(\mathbf{W}_i^{F1}\mathbf{X}_i^S+\mathbf{b}_i^{F1}) +\mathbf{b}_i^{F2}+\mathbf{X}_i^S,
\label{eq:FNN_matrix}
\end{equation}
where $\mathbf{W}_i^F=\{\mathbf{W}_i^{F2},\mathbf{W}_i^{F1}\}$ is the paramter implementation, and $\mathbf{b}_i^{F1}$ and $\mathbf{b}_i^{F2}$ is the bais for the FNN block.

With \textbf{Theorem~\ref{th:right_conv}}, we can transform each part as 
\begin{gather*}
    vec(\mathbf{W}_i^{F1}\mathbf{X}_i^S)=(\mathbf{I}_I \otimes \mathbf{W}_i^{F1}) \mathbf{x}_i^S=\mathbf{W}_i^{F1'} \mathbf{x}_i^S,\\
    \mathbf{b}_i^{F1'}=\mathbf{b}_i^{F1} \otimes \mathbf{1}_I,\\
    \mathbf{b}_i^{F2'}=\mathbf{b}_i^{F2} \otimes \mathbf{1}_I,\\
    vec(\mathbf{W}_i^{F2}\sigma(\cdot))=(\mathbf{I}_I \otimes \mathbf{W}_i^{F2})vec(\sigma(\cdot))=\mathbf{W}_i^{F2'} vec(\sigma(\cdot)).
\end{gather*}

Eq.~(\ref{eq:FNN_matrix}) is then converted into its vector form as
\begin{equation} 
\begin{aligned}
\mathbf{x}_i&=vec(G_i^F(\mathbf{X}_i^S;\mathbf{W}_i^F))=g_i^F(\mathbf{x}_i^S;\mathbf{W}_i^F)\\
&=\mathbf{W}_i^{F2'}\sigma(\mathbf{W}_i^{F1'} \mathbf{x}_i^S+\mathbf{b}_i^{F1'}) +\mathbf{b}_i^{F2'} +\mathbf{x}_i^S.
\end{aligned}
\label{eq:ffn_uat}
\end{equation}

Now we see that the entire procedure matches the form of Eq.~(\ref{eq:res_general_term}), with the only difference that $\mathbf{x}_i^S$ serves as the input in place of $\mathbf{x}_{i-1}$.
In fact, $\mathbf{x}_i^S$ can be viewed as a transformed version of  $\mathbf{x}_{i-1}$. 
This is justified because, although the MHA block is central to a Transformer layer’s performance, computationally it acts as a transformation of the original input $\mathbf{x}_{i-1}$ via the weight matrix $\mathbf{W}_{i}^S$.
Put differently, Transformers form a special class of residual networks: each layer consumes a transformed version of the previous layer’s output, rather than the raw output itself as in standard residual networks.
Consequently, the layer-wise UAT and LPA continue to apply to Transformers.
To ensure greater rigor, we have also conducted a thorough proof of how this difference affects the UAT property in Section \ref{sec:tf_uat}.

\section{A detailed Examination of the Universal Approximation Property in Transformers}
\label{sec:tf_uat}
In this section, we analyze the universal approximation property of Transformers from a mathematical standpoint.

Based on the layer-wise UAT framework in Section~\ref{sec:CFT}, it suffices to show that a single-layer Transformer has universal approximation capability; the result then extends to deeper models.
More specifically, UAT requires three conditions:
\begin{itemize}
    \item \textbf{Compact Input}: the input domain is compact;
    \item \textbf{FFN-centered Network}: the model has the form of UAT with an FFN as its core;
    \item \textbf{Target Validity}: the target function is well-defined on a compact set.
\end{itemize}

We begin by examining the structure of a single Transformer layer and, accordingly, delineating the proof objective.


\subsection{Transformer Sturcture and Proof Objective}
A single-layer Transformer consists primarily of two components: MHA and a FFN, with the overall operation expressed as $(g_i^F \circ g_i^S)(\mathbf{x}_{i-1})$, where: $g_i^S: \mathbb{R}^{dI} \to \mathbb{R}^{dI}$ denotes the feature transformation applied by the MHA module;
$g_i^F: \mathbb{R}^{dI} \to \mathbb{R}^{dI}$ represents the non-linear mapping of the FFN; $\mathbf{x}_{i-1} \in K_{i-1}^I \subset \mathbb{R}^{dI}$ is the input at layer $i$, with $K_{i-1}^I = (K_{i-1})^I$ being the product space of $I$ tokens, each belonging to a compact set $K_{i-1} \subset \mathbb{R}^d$. 
Since finite products of compact sets are compact, $K_{i-1}^I$ is compact in $\mathbb{R}^{dI}$.

Notably, the mathematical form of the FFN component, $g_i^F$ (see Eq.~(\ref{eq:ffn_uat})), aligns with the framework of the UAT (see Eq.~(\ref{eq:res_general_term})), fulfilling the \textbf{FFN-centrered Network} condition.

Let the target function $f_i: K_{i-1}^I \to \mathbb{R}^{dI}$ be continuous. 
The key difficulty in establishing its validity is that, before the FFN $g_i^F$ receives its input $\mathbf{x}_{i}^S$ the MHA block $g_i^S$ applies an additional transformation, and thus the FFN does not act directly on $\mathbf{x}_{i-1}$.

Hence, $f_i$ is valid if there exists a continuous function $f_i^F$ defined on a compact domain such that
$$
f_i(\mathbf{x}_{i-1}) = (f_i^F \circ g_i^S)(\mathbf{x}_{i-1}),
$$
and $f_i^F$ can be approximated by $g_i^F$. 
Equivalently, $f_i^F$ serves as the transformed target function for the FFN module.
We can now state our proof goals:
\begin{enumerate}
    \item \textbf{Compact Input}: $K_i^S := g_i^S(K_{i-1}^I)$ is a compact subset of $\mathbb{R}^{dI}$;
    \item \textbf{Target Validty}: There exists a continuous function $f_i^F: K_i^S \to \mathbb{R}^{dI}$ such that $f_i = f_i^F \circ g_i^S$.
\end{enumerate}

\subsection{$K_i^S$ is Compact}
\begin{proof}
Since $K_{i-1}^I$ is a finite product of compact sets, it is compact in $\mathbb{R}^{dI}$. The transformation $g_i^S$ is composed of matrix multiplication, Kronecker product, Softmax, and vector addition. The Softmax function is continuous on $\mathbb{R}^{I \times I}$, and all linear operations are continuous; compositions and combinations of continuous functions remain continuous. Hence, $g_i^S$ is a continuous function of the input $\mathbf{x}_{i-1}$. As the continuous image of a compact set is compact, the set
$$
K_i^S := g_i^S(K_{i-1}^I) = \left\{ g_i^S(\mathbf{x}_{i-1}) \mid \mathbf{x}_{i-1} \in K_{i-1}^I \right\}
$$
is compact in $\mathbb{R}^{dI}$.
\end{proof}

\subsection{$f_i^F$ is Valid}
\subsubsection{$f_i$ is decomposable} 
The first requirement for $f_i^F$ to be well-defined is that $f_i$ be decomposable.
By the Factorization Continuity Theorem (Theorem~\ref{th:FCT}), this holds provided the following condition is satisfied:
\begin{equation}
    g_i^S(\mathbf{x}_i^{(1)}) = g_i^S(\mathbf{x}_i^{(2)}) \implies f_i(\mathbf{x}_i^{(1)}) = f_i(\mathbf{x}_i^{(2)}), \quad \forall \mathbf{x}_i^{(1)}, \mathbf{x}_i^{(2)} \in K_{i-1}^I.
    \label{eq:f_dis_T}
\end{equation}

We regard this condition as satisfied because it is a learnable property reinforced by optimization.
The rationale is as follows.
If $g_i^S$ maps two inputs $\mathbf{x}_i^{(1)}, \mathbf{x}_i^{(2)}$ with $f_i(\mathbf{x}_i^{(1)}) \ne f_i(\mathbf{x}_i^{(2)})$ to the same representation $g_i^S(\mathbf{x}_i^{(1)}) = g_i^S(\mathbf{x}_i^{(2)})$, then the FFN (being a deterministic function) cannot simultaneously output both $f_i(\mathbf{x}_i^{(1)})$ and $f_i(\mathbf{x}_i^{(2)})$, resulting in irreducible approximation error and high loss. Consequently, gradient descent penalizes such ``representation collapse", encouraging $g_i^S$ to preserve distinctions between inputs with different target values. In other words, the attention mechanism, combined with training, promotes the emergence of consistency representations for $f_i$, ensuring that inputs differing under $g_i^S$ are not collapsed into the same latent point.

\subsubsection{Existence, continuity, and approximability of $f_i^F$}
\begin{proof}
Under condition in Eq.~(\ref{eq:f_dis_T}), we define $f_i^F: K_i^S \to \mathbb{R}^{dI}$ by
$$
f_i^F(\mathbf{z}) := f_i(\mathbf{x}), \quad \text{for any } \mathbf{x} \in K_{i-1}^I \text{ such that } g_i^S(\mathbf{x}) = \mathbf{z}.
$$
Condition in Eq.~(\ref{eq:f_dis_T}) ensures that $f_i^F$ is well-defined. By the Factorization Continuity Theorem~\ref{th:FCT}, $f_i^F$ is continuous on the compact set $K_i^S$.  

Then, by the UAT, there exists a FFN $g_i^F$ that uniformly approximates $f_i^F$ to arbitrary precision.

\end{proof}

In summary, the MHA module $g_i^S$ continuously maps the input space $K_{i-1}^I$ to a latent representation space $K_i^S$, while the training dynamics encourage $g_i^S$ to maintain consistency representation, ensuring that $f_i$ can be factored as $f_i^F \circ g_i^S$. The FFN $g_i^F$ then approximates $f_i^F$, completing the approximation. Therefore, a single-layer Transformer is a universal approximator.

\section{Representation Consistency and Optimization Dynamics in Deep Residual Networks}
\label{apd:sec:optimization_dynamics}

A multi-layer RN can be viewed as an approximation to a target function $f: \mathbb{R}^n \to \mathbb{R}^m$ with a series of single-layer RNs $\hat{f}_N, \hat{f}_{N-1}, \cdots , \hat{f}_1$:
$$
\hat{f} = \hat{f}_N \circ \hat{f}_{N-1} \circ \cdots \circ \hat{f}_1,
$$
During the forward pass, an intermediate module $\hat{f}_i$ computes a representation of the input and hands it off to subsequent layers $\hat{f}_{i+1}, ..., \hat{f}_{N}$ for further processing. These intermediate representations are crucial for accurately approximating the target function $f$. A key requirement is that $\hat{f}_i$ produce sufficiently distinguishable representations.


\subsection{Distinguishability Preservation in Intermediate Representations}
In RNs, the effectiveness of intermediate representations depends not only on hierarchical abstraction, but more critically on their ability to provide embeddings consitent to the outputs of the target function $f$. Consider an arbitrary intermediate layer $i$ ($1 \leq i \leq N$), and define the encoding path up to layer $i-1$ as:
$$
\phi_i = \hat{f}_{i-1} \circ \cdots \circ \hat{f}_1,
$$
which maps the input $\mathbf{x} \in \mathbb{R}^n$ to its representation at the output of the $i-1$-th layer. The decoding path from layer $i$ to the output is defined as:
$$
g_i = \hat{f}_N \circ \cdots \circ \hat{f}_i.
$$
Thus, the network computes $\hat{f} = g_i \circ \phi_i$, where $\hat{f}$ is an approximation to the target function $f$.

Assume there are inputs $\mathbf{x}_0^{(1)}, \mathbf{x}_0^{(2)}$ with $f(\mathbf{x}_0^{(1)}) \ne f(\mathbf{x}_0^{(2)})$, yet $\phi_i(\mathbf{x}_0^{(1)}) = \phi_i(\mathbf{x}_0^{(2)})$. 
Then, due to the deterministic nature of $g_i$, the network will produce identical final estimates, $\hat{f}(\mathbf{x}_0^{(1)}) = \hat{f}(\mathbf{x}_0^{(2)})$, contradicting the ground truth $f(\mathbf{x}_0^{(1)}) \ne f(\mathbf{x}_0^{(2)})$.
This is an indication that the optimization of the network has not converged to its optimal.

Therefore, by Factorization Continuity Theorem
\ref{th:FCT}, upon convergence (idealy $\hat{f} = f$), $\phi_i$ must satisfy:
\begin{equation}
\phi_i(\mathbf{x}_0^{(1)}) = \phi_i(\mathbf{x}_0^{(2)}) \implies f(\mathbf{x}_0^{(1)}) = f(\mathbf{x}_0^{(2)}), \quad \forall\, \mathbf{x}_0^{(1)}, \mathbf{x}_0^{(2)} \in \mathbb{R}^n.
\label{apd:eq:rep_consistency}
\end{equation}
This condition is a necessary condition for $f$ can be represented by $\hat{f}$: if violated, $f$ lies outside the expressive capacity of the $f$.

In practice, while $\hat{f}$ is typically only required to approximate $f$ rather than match it exactly, significant issues still arise when $\phi_i(\mathbf{x}_0^{(1)}) \approx \phi_i(\mathbf{x}_0^{(2)})$ but $f(\mathbf{x}_0^{(1)})$ and $f(\mathbf{x}_0^{(2)})$ differ substantially. In such cases, the decoder $g_i$ must generate vastly different outputs from nearly identical inputs, demanding extreme sensitivity to small input perturbations. This not only increases the difficulty of approximation but also risks training instability. Hence, even in approximate learning, intermediate representations should avoid mapping inputs with significantly different targets into overly similar regions of the representation space.

Equivalently, this principle can be expressed more directly by taking the conv erse-negative proposition of Eq.~(\ref{apd:eq:rep_consistency}) as:
$$
f(\mathbf{x}_0^{(1)}) \ne f(\mathbf{x}_0^{(2)}) \implies \phi_i(\mathbf{x}_0^{(1)}) \ne \phi_i(\mathbf{x}_0^{(2)}), \quad \forall\, \mathbf{x}_0^{(1)}, \mathbf{x}_0^{(2)} \in \mathbb{R}^n.
$$
We refer to this necessary condition as $\hat{f}$-distinguishability, it means that the intermediate representation $\phi_i$ must preserve the distinguishability of inputs with distinct true target values under $f$. Notably, the condition is asymmetric: when $f(\mathbf{x}_0^{(1)}) = f(\mathbf{x}_0^{(2)})$, multiple inputs may share similar or even identical representations. Such over-separation is harmless and often beneficial, as it allows the network to encode auxiliary information that may aid generalization.

In summary, maintaining distinguishability between intermediate representations with distinct target is a key design principle for efficient approximation. The most critical failure mode is false merging, in which inputs with different targets are irreversibly collapsed into the same or highly similar representations. Once such structural loss occurs, no subsequent layer can recover the necessary distinctions, rendering the error uncorrectable.

\subsection{Unified Interpretation via Optimization Dynamics: Emergence of $f$-Distinguishability}
\label{sec:uni_dis}

Under both E2E and LPA optimization, representation collapse induces irreducible error. Since gradient descent inherently minimizes the loss, it implicitly penalizes parameter configurations that lead to information loss. This mechanism drives the network to gradually establish $f$-consistency in intermediate representation during training—inputs with different $f$-values become increasingly separated in the intermediate representation space.

Notably, the residual connection plays a crucial role: by preserving the identity path, it maintains access to the original input information, reducing the risk of information loss and providing structural support for the emergence of $f$-consistency.

\subsection{Generality: Optimization-Induced Separation}
\label{sec:general-dis}
This mechanism is not limited to residual networks or specific optimization strategies. In any deep model where subsequent modules depend deterministically on earlier representations (i.e., $\mathbf{x}_j = T_j(\mathbf{x}_{j-1})$), representational collapse leads to unrecoverable errors. Thus, both end-to-end joint training and layer-wise optimization use loss feedback to drive the network away from information merging.

We term this universal phenomenon optimization-induced consistency: the training process itself acts as an implicit regularizer, encouraging the network to preserve task-relevant information structures during representation learning. This provides a unified perspective for understanding generalization, representation evolution, and universal approximation in deep networks.

\section{Factorization Continuity Theorem}
\label{sec:Factorization Continuity Theorem}
\begin{theorem}[\textbf{Factorization and Continuity}]
\label{th:FCT}
Let $\mathbf{I}_n \subset \mathbb{R}^n$ be compact, $f: \mathbf{I}_n \to \mathbb{R}^m$ continuous, and $T: \mathbf{I}_n \to Z$ continuous with $Z \subset \mathbb{R}^d$ compact. 
Then:
\begin{itemize}
    \item \textbf{(Existence of $g$)} There exists a function $g: T(\mathbf{I}_n) \to \mathbb{R}^m$ such that $f = g \circ T$ if and only if
   $$
   T(\mathbf{x}^{(1)}) = T(\mathbf{x}^{(2)}) \implies f(\mathbf{x}^{(1)}) = f(\mathbf{x}^{(2)}), \quad \forall \mathbf{x}^{(1)}, \mathbf{x}^{(2)} \in \mathbf{I}_n.
   $$
   \item \textbf{(Continuity of $g$)} If such a $g$ exists, then $g$ is continuous on the compact set $T(\mathbf{I}_n)$.
\end{itemize}
\end{theorem}

If the above assumption holds, then since $g$ is continuous on the compact set $T(\mathbf{I}_n) \subseteq Z$, it follows that $g$ can be approximated by a neural network. We now provide the proof.

\begin{proof}
Let $\mathbf{I}_n \subseteq \mathbb{R}^n$ be a nonempty compact set, $f: \mathbf{I}_n \to \mathbb{R}^m$ a continuous function, and $T: \mathbf{I}_n \to \mathbb{R}^d$ a continuous mapping. Define the image set:
$$
Z := T(\mathbf{I}_n) = \{ T(\mathbf{x}) \mid \mathbf{x} \in \mathbf{I}_n \} \subseteq \mathbb{R}^d.
$$
Since $\mathbf{I}_n$ is compact and $T$ is continuous, $Z$ is also compact (the continuous image of a compact set is compact).

Our purpose is to prove that there exists a function $g: Z \to \mathbb{R}^m$ such that:
$$
f(\mathbf{x}) = g(T(\mathbf{x})) = (g \circ T)(\mathbf{x}), \quad \forall \mathbf{x} \in \mathbf{I}_n,
$$
i.e., $f = g \circ T$.
\end{proof}

\begin{lemma}[\textbf{Existence}]
Let $T: \mathbf{I}_n \to \mathbb{R}^d$ be an arbitrary mapping and $f: \mathbf{I}_n \to \mathbb{R}^m$. If 
$$
\forall \mathbf{x}^{(1)}, \mathbf{x}^{(2)} \in \mathbf{I}_n, \quad T(\mathbf{x}^{(1)}) = T(\mathbf{x}^{(2)}) \implies f(\mathbf{x}^{(1)}) = f(\mathbf{x}^{(2)}),
$$
then there exists a function $g: T(\mathbf{I}_n) \to \mathbb{R}^m$ such that $f = g \circ T$.
\end{lemma}

\begin{proof}

For each $\mathbf{z} \in Z$, define $g(\mathbf{z}) := f(\mathbf{x})$ for any $\mathbf{x} \in \mathbf{I}_n$ with $T(\mathbf{x}) = \mathbf{z}$. Such an $\mathbf{x}$ exists because $Z = T(\mathbf{I}_n)$. This definition is well-defined: if $T(\mathbf{x}^{(1)}) = T(\mathbf{x}^{(2)}) = \mathbf{z}$, then by assumption $f(\mathbf{x}^{(1)}) = f(\mathbf{x}^{(2)})$. So $g(\mathbf{z})$ is independent of the choice of preimage.

By construction, $g(T(\mathbf{x})) = f(\mathbf{x})$ for all $\mathbf{x} \in \mathbf{I}_n$, so there exists a function $g$ fulfilling $f = g \circ T$.

\textbf{Necessity:} Suppose $f = g \circ T$, i.e., $f(\mathbf{x}) = g(T(\mathbf{x}))$. If $T(\mathbf{x}^{(1)}) = T(\mathbf{x}^{(2)})$, then:
$$
f(\mathbf{x}^{(1)}) = g(T(\mathbf{x}^{(1)})) = g(T(\mathbf{x}^{(2)})) = f(\mathbf{x}^{(2)}).
$$
This proves the necessity.

\textbf{Sufficiency:} Suppose that 
$$
T(\mathbf{x}^{(1)}) = T(\mathbf{x}^{(2)}) \implies f(\mathbf{x}^{(1)}) = f(\mathbf{x}^{(2)}), \quad \forall \mathbf{x}^{(1)}, \mathbf{x}^{(2)} \in \mathbf{I}_n.
$$
We aim to construct a function $g: Z := T(\mathbf{I}_n) \to \mathbb{R}^m$ such that $f = g \circ T$. Define $g$ as follows: for each $\mathbf{z} \in Z$, choose any $\mathbf{x} \in \mathbf{I}_n$ such that $T(\mathbf{x}) = \mathbf{z}$ (such $\mathbf{x}$ exists by definition of $Z$), and set
$$
g(\mathbf{z}) := f(\mathbf{x}).
$$

We now show that $g$ is well-defined, i.e., the value $g(\mathbf{z})$ does not depend on the choice of preimage $\mathbf{x}$. Let $\mathbf{x}^{(1)}, \mathbf{x}^{(2)} \in \mathbf{I}_n$ be such that $T(\mathbf{x}^{(1)}) = T(\mathbf{x}^{(2)}) = \mathbf{z}$. Then by assumption, $f(\mathbf{x}^{(1)}) = f(\mathbf{x}^{(2)})$. Hence, regardless of which preimage we choose, $g(\mathbf{z}) = f(\mathbf{x})$ is the same. Therefore, $g$ is well-defined. Finally, by construction, for every $\mathbf{x} \in \mathbf{I}_n$, we have
$$
g(T(\mathbf{x})) = f(\mathbf{x}),
$$
since $g(T(\mathbf{x}))$ was defined precisely as $f(\mathbf{x})$ when $\mathbf{x}$ was chosen as the preimage of $T(\mathbf{x})$. Thus, $f = g \circ T$.
\end{proof}

\begin{lemma}[\textbf{Uniqueness}] 
If $g_1, g_2: Z := T(\mathbf{I}_n) \to \mathbb{R}^m$ both satisfy $f = g_1 \circ T = g_2 \circ T$, then $g_1 = g_2$ on $Z$.
\end{lemma}

\begin{proof}
Assume $g_1$ and $g_2$ satisfy:
$$
f(\mathbf{x}) = g_1(T(\mathbf{x})), \quad f(\mathbf{x}) = g_2(T(\mathbf{x})), \quad \forall \mathbf{x} \in \mathbf{I}_n.
$$
We aim to show $g_1(\mathbf{z}) = g_2(\mathbf{z})$ for all $\mathbf{z} \in Z$. Let $\mathbf{z} \in Z$ be arbitrary. Since $Z = T(\mathbf{I}_n)$, there exists $\mathbf{x} \in \mathbf{I}_n$ such that $T(\mathbf{x}) = \mathbf{z}$. Then:
$$
g_1(\mathbf{z}) = g_1(T(\mathbf{x})) = f(\mathbf{x}), \\
g_2(\mathbf{z}) = g_2(T(\mathbf{x})) = f(\mathbf{x}).
$$
Hence $g_1(\mathbf{z}) = f(\mathbf{x}) = g_2(\mathbf{z})$. Since $\mathbf{z}$ was arbitrary, $g_1 = g_2$ on $Z$. Therefore, the function $g: Z \to \mathbb{R}^m$ satisfying $f = g \circ T$ is unique.
\end{proof}

\begin{lemma}[\textbf{Continuity}]   
Let $\mathbf{I}_n \subseteq \mathbb{R}^n$ be a nonempty compact set, $f: \mathbf{I}_n \to \mathbb{R}^m$ continuous, $T: \mathbf{I}_n \to \mathbb{R}^d$ continuous, and suppose:
$$
\forall \mathbf{x}^{(1)}, \mathbf{x}^{(2)} \in \mathbf{I}_n, \quad T(\mathbf{x}^{(1)}) = T(\mathbf{x}^{(2)}) \implies f(\mathbf{x}^{(1)}) = f(\mathbf{x}^{(2)}).
$$
Let $Z := T(\mathbf{I}_n) \subseteq \mathbb{R}^d$. Then there exists a unique function $g: Z \to \mathbb{R}^m$ such that $f = g \circ T$. We claim that $g$ is continuous on $Z$.
\end{lemma}

\begin{proof} 
Suppose, for contradiction, that $g$ is discontinuous at some point $\mathbf{z}_0 \in Z$. Then there exists $\varepsilon_0 > 0$ and a sequence $\{\mathbf{z}_k\} \subset Z$ such that $\mathbf{z}_k \to \mathbf{z}_0$, but
$$
\|g(\mathbf{z}_k) - g(\mathbf{z}_0)\| \geq \varepsilon_0, \quad \forall k \in \mathbb{N}.
$$
Since $\mathbf{z}_k \in Z = T(\mathbf{I}_n)$, for each $k$, there exists $\mathbf{x}_k \in \mathbf{I}_n$ such that $T(\mathbf{x}_k) = \mathbf{z}_k$. As $\mathbf{I}_n$ is compact, the sequence $\{\mathbf{x}_k\}$ has a convergent subsequence $\{\mathbf{x}_{k_j}\}$ such that
$$
\mathbf{x}_{k_j} \to \mathbf{x}^* \in \mathbf{I}_n, \quad \text{as } j \to \infty.
$$
By continuity of $T$, we have
$$
T(\mathbf{x}_{k_j}) \to T(\mathbf{x}^*).
$$
But $T(\mathbf{x}_{k_j}) = \mathbf{z}_{k_j}$, and since $\mathbf{z}_k \to \mathbf{z}_0$, the subsequence $\mathbf{z}_{k_j} \to \mathbf{z}_0$. By uniqueness of limits,
$$
T(\mathbf{x}^*) = \mathbf{z}_0.
$$
By continuity of $f$,
$$
f(\mathbf{x}_{k_j}) \to f(\mathbf{x}^*).
$$
Note that $f(\mathbf{x}_{k_j}) = g(T(\mathbf{x}_{k_j})) = g(\mathbf{z}_{k_j})$, and $f(\mathbf{x}^*) = g(T(\mathbf{x}^*)) = g(\mathbf{z}_0)$. Therefore,
$$
g(\mathbf{z}_{k_j}) \to g(\mathbf{z}_0), \quad \text{as } j \to \infty.
$$
Thus, there exists $J \in \mathbb{N}$ such that for all $j \geq J$, $\|g(\mathbf{z}_{k_j}) - g(\mathbf{z}_0)\| < \varepsilon_0$. However, this contradicts the earlier assumption that
$$
\|g(\mathbf{z}_k) - g(\mathbf{z}_0)\| \geq \varepsilon_0, \quad \forall k,
$$
which in particular implies $\|g(\mathbf{z}_{k_j}) - g(\mathbf{z}_0)\| \geq \varepsilon_0$ for all $j$. Hence, the assumption of discontinuity is false. Therefore, $g$ is continuous at $\mathbf{z}_0$. Since $\mathbf{z}_0 \in Z$ was arbitrary, $g$ is continuous on $Z$.
\end{proof}

\section{Experimental Details}
\label{sec:exp_details}

To ensure reproducibility and address potential concerns regarding experimental completeness, we provide full implementation details below.

\paragraph{Synthetic Simulation.}
We generate 10,000 random input--output pairs as the training set and an additional 10,000 uniformly sampled pairs as the test set. The model is a 6-layer residual MLP with hidden dimension 30. We train for 300 epochs using standard optimization settings.

\paragraph{Image Classification.}
We evaluate on CIFAR-10, CIFAR-100, and Fashion-MNIST using both ViT and ResNet architectures.
\begin{itemize}
    \item For ViT, we use 12- and 24-layer variants with attention heads = 3, embedding dimension = 192, and MLP hidden size = 768, resulting in approximately 5.5M and 10.9M parameters, respectively.
    \item For ResNet, we adopt standard ResNet-18 and ResNet-34. Following the analysis in Appendix~\ref{sec:dim_change_proof}, our LPA framework naturally incorporates dimension-matching projections via the shared readout head, eliminating the need for additional ad-hoc projection modules.
\end{itemize}
All image classification models are trained for 300 epochs. We use a learning rate of $6 \times 10^{-4}$ for ResNet and $1 \times 10^{-4}$ for ViT.

\textbf{Train ImageNet From-scratch:} To validate our method under large-scale training, we train ViT models from scratch on ImageNet-1k for 300 epochs with an initial learning rate of $5 \times 10^{-4}$. The input image size is $224 \times 224$ with patch size 16. Specifically:
\begin{itemize}
    \item \textbf{ViT-12}: 12 layers, embedding dimension 768, 12 attention heads, FFN hidden dimension 3072.
    \item \textbf{ViT-24}: 24 layers, embedding dimension 1024, 16 attention heads, FFN hidden dimension 4096.
\end{itemize}

\textbf{Finetuning from pretrained checkpoints.}  
We finetune from publicly available DeiT pretrained weights on ImageNet-1k:
\begin{itemize}
    \item \textbf{ViT-Base}~\footnote{\url{https://dl.fbaipublicfiles.com/deit/deit_base_patch16_224-b5f2ef4d.pth}}: 12 layers, embedding dimension 768, 12 attention heads, FFN hidden dimension 3072.
    \item \textbf{ViT-Small}~\footnote{\url{https://dl.fbaipublicfiles.com/deit/deit_small_patch16_224-cd65a155.pth}}: 24 layers, embedding dimension 1024, 16 attention heads, FFN hidden dimension 4096.
\end{itemize}

We finetune for 30 epochs with a learning rate of $1 \times 10^{-5}$ (same as used in the main ViT experiments above).

\paragraph{Large Language Model Tasks.}
Due to computational constraints, we do not train LLMs from scratch. Instead, we initialize from the pretrained Qwen2-0.5B model (24 layers).
\begin{itemize}
    \item For classification benchmarks (CoLA, MNLI, MRPC, QNLI, QQP, SST-2), which impose relatively modest demands on model capacity compared to generation tasks, we apply LPA supervision to \textbf{all layers}.
    \item For generation tasks (Wikitext, Alpaca, Dolly-15k, GSM8K), which are more complex, we apply LPA supervision only to layers 20--24 (i.e., the top 5 layers).
\end{itemize}

In all LPA training setups, we assign a weight of 0.1 to losses from earlier layers and a weight of 1.0 to the final layer. This weighting scheme is motivated by the structural property of deep residual networks: as shown in Eq.~\eqref{eq:layer_rep_shh} and its generalization in Eq.~\eqref{eq:layer_rep_general}, the final prediction is a cumulative sum of layer-wise contributions, with the last layer directly shaping the output. Emphasizing the final layer ensures that the primary learning objective remains aligned with the E2E task performance, while auxiliary supervision on earlier layers provides regularization without compromising convergence.

We also note that LPA incurs negligible computational overhead compared to standard E2E training. Since LPA only reuses the final linear readout head at intermediate layers, adding a lightweight linear projection per layer, the additional cost is minimal relative to the full network computation. For example, on CIFAR-10 with a ViT-12 model trained on an NVIDIA A6000 GPU, both LPA and E2E take approximately 28 seconds per epoch.

\end{document}